\documentclass{article}
\PassOptionsToPackage{numbers, sort&compress}{natbib}
\usepackage[dblblindworkshop,final]{neurips_2025}
\workshoptitle{2nd Workshop on Principles of Generative Modeling (PriGM)}
\makeatletter
\renewcommand{\@neuripsyear}{2026}
\renewcommand{\@neuripsordinal}{40th}
\makeatother
\usepackage[T1]{fontenc}
\usepackage[utf8]{inputenc}

\usepackage{amsmath,amssymb,amsthm,booktabs,graphicx,longtable,microtype,float}
\usepackage{url}
\usepackage{hyperref}
\hypersetup{colorlinks=false,pdfborder={0 0 1},
  citebordercolor={0 1 0},urlbordercolor={0 1 1},linkbordercolor={1 0 0}}
\usepackage{orcidlink}
\hypersetup{pdfauthor={Kevin Lee},pdftitle={An Input Skip Can Reparameterize a VAE Without Improving the Generator},pdfsubject={},pdfkeywords={}}
\title{An Input Skip Can Reparameterize a VAE Without Improving the Generator}
\author{%
 Kevin Lee{\hypersetup{pdfborder={0 0 0}}\orcidlink{0009-0004-0388-9260}} \\
 [4pt] University of California at Los Angeles \\
 NASA Jet Propulsion Laboratory \\
 \texttt{kevinlee69720@g.ucla.edu}
}
\newcommand{\E}{\mathbb{E}}
\newcommand{\KL}{D_{\mathrm{KL}}}
\newtheorem{proposition}{Proposition}

\begin{document}
\maketitle

\begin{abstract}
An input skip can reduce reconstruction error while leaving a model's distribution of new samples unchanged. For a Gaussian variational autoencoder (VAE) with constant input skip $0\leq\alpha<1$, a model without a skip can be written in skip form by multiplying its decoder output by $b=1-\alpha$ and its noise variance by $b^2$. The distribution of observations and latent variables and its exact generative scores stay the same, while reconstruction mean squared error (MSE) is multiplied by $b^2$. Separate training can produce different generators because optimization depends on the parameter coordinates. In a Gaussian experiment with known density, changing skip and observation variance together reduces MSE but worsens energy score and the estimate of negative log density from importance sampling. On seven annual \mbox{Deep-ocean} Assessment and Reporting of Tsunamis (DART) records, energy score differences at the same observation variance have both signs on four original records and favor the skip on three records from 2018. Probabilistic principal component analysis (PPCA) has lower mean energy scores than every neural DART fit. With the encoder and decoder kept the same, estimating observation variance from training reconstruction error can also worsen energy score. Assessing generation requires more than reconstruction error.
\end{abstract}

\section{Introduction}
A variational autoencoder (VAE) reconstructs an observation through an encoder and a decoder. To generate a new observation, it samples a latent variable $z$ from the prior and then samples from the decoder $p(x\mid z)$ \citep{KingmaWelling2014AutoEncodingVB}. An input skip adds part of the observation directly to the reconstruction: $\hat x=f_\theta(z)+\alpha x$. Increasing $\alpha$ lets the input account for more of that reconstruction. This can reduce reconstruction error while generative performance stays the same or worsens.\footnote{Available at: \url{https://github.com/magnaprog/vae-skip-reparam}.}

For a constant skip and Gaussian noise, rescaling the decoder and noise preserves the generator while multiplying MSE by $(1-\alpha)^2$. Separate training can change the generator because Adam steps and gradient clipping depend on parameter scaling \citep{KingmaBa2015Adam,Dinh2017SharpMinima}. Comparisons at the same observation variance assess these training outcomes. Dieng et al.\ use a different skip: their links carry $z$ into decoder layers to encourage dependence on the latent variable \citep{Dieng2019SkipVAE}.

National Oceanic and Atmospheric Administration (NOAA) Deep-ocean Assessment and Reporting of Tsunamis (DART) records provide an application in which both oscillations and changes in mean level matter. Their HEIGHT variable expresses bottom pressure as equivalent water column height \citep{NDBCDARTDescription,Meinig2005DARTII}. Earlier work used an adaptive input skip and iterative reconstruction at station 23461 in 2022 \citep{Lee2024DARTPredecessor}. Here a constant skip isolates the effect of feeding the input into reconstruction. Models generate windows of 192 consecutive observations for comparison with later data. Observation variance is varied separately because it sets the relative weight of reconstruction and Kullback--Leibler (KL) divergence during training \citep{Rybkin2021CalibratedDecoders}.

\section{Input skip and generative evaluation}
Let $x\in\mathbb{R}^d$, $0\leq\alpha<1$, and $b=1-\alpha$. The encoder $q_\phi(z\mid x)$ approximates the posterior and supplies latent samples for importance sampling. Reconstruction MSE is $d^{-1}\|x-f_\theta(\E_q Z)-\alpha x\|^2$. If $f_\theta=b g_\theta$, then $x-f_\theta(z)-\alpha x=b(x-g_\theta(z))$, so squared error scales by $b^2$, equal to $0.04$ at $\alpha=0.8$. Dividing reconstruction MSE by $b^2$ gives MSE in observation coordinates.

The implicit generative model is $X=\alpha X+f_\theta(Z)+\epsilon$, where $Z\sim\mathcal{N}(0,I)$ and $\epsilon\sim\mathcal{N}(0,\sigma_r^2I)$ are independent. Solving for $X$ gives its conditional density and a rule for generating samples:
\begin{equation}
 p_\alpha(x\mid z)=b^d\mathcal{N}(bx;f_\theta(z),\sigma_r^2I),
 \qquad X=\frac{f_\theta(Z)+\epsilon}{b}.
\label{eq:density}
\end{equation}
Thus, the conditional observation mean is $g_\theta(z)=f_\theta(z)/b$ and observation variance is $\sigma_x^2=\sigma_r^2/b^2$. The evidence lower bound (ELBO) is a variational lower bound on $\log p_\alpha(x)$.

\begin{proposition}
\label{prop:equivalence}
Consider a model with conditional density $\mathcal{N}(x;g_\theta(z),\sigma_0^2I)$, where $\sigma_0^2>0$, prior $p(z)=\mathcal{N}(0,I)$, and encoder $q_\phi(z\mid x)$. Set the residual decoder to $f=b\,g_\theta$ and its variance to $\sigma_r^2=b^2\sigma_0^2$, where $b>0$. The distribution of $(X,Z)$ is unchanged, as are the marginal $p_\alpha(x)$ and exact posterior $p_\alpha(z\mid x)$. Using the same encoder preserves the ELBO; using the same latent samples also preserves each importance weight $p(x\mid z)p(z)/q_\phi(z\mid x)$. Reconstruction MSE scales by $b^2$.
\end{proposition}

The proof uses the change of variables $x$ to $bx$ with its Jacobian factor (Appendix A). Table~\ref{tab:conditions} separates skip and observation variance. Separate fits share Adam settings and the gradient norm limit in residual coordinates, permitting different training paths even from equal initial observation means.

\begin{table}[t]
\centering
\caption{Skip and variance conditions. The high variance is $25\sigma_0^2$, where $25=1/(1-0.8)^2$. The original comparison changes both factors, from no skip at base variance to skip at high variance.}
\label{tab:conditions}
\begin{tabular}{lcc}
\toprule
Skip coefficient & Base variance: $\sigma_0^2$ & High variance: $25\sigma_0^2$\\
\midrule
$\alpha=0$ & No skip, base variance & No skip, high variance\\
$\alpha=0.8$ & Skip, base variance & Skip, high variance\\
\bottomrule
\end{tabular}
\end{table}

For each observation, training uses one sample from the encoder distribution to estimate the expectation in the negative ELBO,

\begin{equation}
\mathcal{L}_\alpha
=-\E\bigl[\log p_\alpha(x\mid z)\bigr]+\KL\bigl(q_\phi(z\mid x)\Vert p(z)\bigr),
\qquad z\sim q_\phi(z\mid x),
\label{eq:objective}
\end{equation}

using the density in Eq.~\ref{eq:density}, including its normalization, and KL coefficient $\beta=1$. Appendix A derives the corresponding squared reconstruction error and KL terms.

\section{Experiments}
\subsection{Data}
The Gaussian experiment has eight independent normal latent factors, orthogonal sinusoidal loadings in 192 coordinates, and isotropic noise of standard deviation 0.25. NOAA describes instrument measurements every 15 s and routine T=1 values 15 minutes apart \citep{NDBCDARTDescription,Meinig2005DARTII,NDBC2023WebDataGuide}. The analysis uses T=1 observations from NOAA annual files \citep{NDBCDARTHistoricalData}, without despiking, detrending, or detiding. Each window has 192 samples at 900 s spacing, spanning 47 h 45 min.

Four original records use equal window weights; three training partitions span deployments. Three records from 2018 use overlapping test windows within NOAA deployment periods. The choice of 44402/2018 followed examination of its 2016 file. Stations 21415 and 21418 were chosen in identifier order among eligible records before neural results. Their windows start every four hours in Coordinated Universal Time (UTC) with calendar 60/20/20 splits and exclusions of two days on each side of each split date. Each of six UTC start hours has total weight $1/6$. The appendix gives station records, counts, and selection criteria.

\subsection{Models and comparisons}
Each record uses four conditions and three training seeds. Within each 2018 record and seed, conditions share initial observation means, minibatches, and encoder sampling noise. Encoders use 16 latent coordinates; PPCA has at most 16 loading directions \citep{TippingBishop1999ProbabilisticPCA}. PPCA determines base variance, using a validation grid for 2018 before neural fitting. Original models use weights chosen by validation. All 36 models for 2018 complete 10,000 Adam steps, with evaluation of final weights and weights chosen by validation.

\subsection{Evaluation}
The energy score measures agreement between a model distribution and observations through Euclidean distances \citep{GneitingRaftery2007ProperScoringRules}. The name comes from energy statistics and an analogy with potential energy \citep{RizzoSzekely2016EnergyDistance}. For each observation, the score is the mean distance from model samples to that observation minus half the mean distance between independent model samples. Lower scores indicate better agreement. For DART height $h$, the transformation $u=(h-m)/s$ uses the training mean $m$ and standard deviation $s$ to give dimensionless coordinates. Scores are computed in these coordinates for full windows, for windows after subtraction of their own means, and for time differences across several lags. Mean subtraction removes the level of each window; time differences measure changes over the chosen lags. DART evaluation uses three independent streams of 1,024 generated windows, with the same latent prior samples and standard normal noise across conditions. Each comparison pairs conditions with the same training seed. Medians and full ranges across three seeds describe variation among fits; ranges across sampling streams show Monte Carlo variation. Negative importance weighted (IW) log density estimates use $K=256$ and three streams \citep{BurdaGrosseSalakhutdinov2016IWAE}. Low effective sample size (ESS) leaves their accuracy uncertain.

\subsection{Observation variance estimation}
With encoder and decoder weights unchanged, observation variance is estimated as
\begin{equation}
\hat v=\frac{1}{d}\sum_{i\in\mathrm{train}}w_i\,\widehat{\E}_{q_\phi(z\mid x_i)}\|x_i-g_\theta(z)\|^2
\label{eq:variance_refit}
\end{equation}
using three independent streams of 32 samples from the encoder per training window and training weights that sum to one. For these samples, the estimate minimizes the expected Gaussian reconstruction term, including the log normalizer. Rybkin et al.\ optimize encoder and decoder weights as well as observation variance \citep{Rybkin2021CalibratedDecoders}. In this calculation the network weights stay the same, and the new variance is used for evaluation and sample generation.

\section{Results}
\subsection{Synthetic Gaussian experiment}
Changing both skip and observation variance, from no skip at base variance to skip at high variance, reduces median reconstruction MSE from 0.0616 to 0.00368 across three runs. Over the same comparison, median estimated negative IW log density increases from 0.133 to 1.201 nats per coordinate, and median energy score increases from 5.723 to 7.123. Both generative scores worsen as reconstruction error falls. At base observation variance, the skip model has median MSE 0.00250 and median negative log estimate 0.137; the corresponding exact negative log scores are 0.125 for the known generating Gaussian and 0.127 for PPCA. These results come from separate fits. Rescaling the decoder and noise of one fitted model to introduce the skip at base variance would preserve its generator and exact proper scores.

The high observation variance sets a noise floor before training. Its value, $25\sigma_0^2\approx1.53929$, exceeds the data variance along 185 of 192 covariance directions. Since independent observation noise gives $\operatorname{Cov}(X)=\operatorname{Cov}(g(Z))+\sigma_x^2I\succeq\sigma_x^2I$, those models necessarily have excess variance in these directions. The relation $\succeq$ means that the matrix difference is positive semidefinite. Increasing decoder flexibility leaves this noise contribution in place. Score differences also depend on fitting and, for log estimates, approximate inference.

\begin{figure}[t]
\centering
\includegraphics[width=\linewidth,height=0.81\textheight,keepaspectratio]{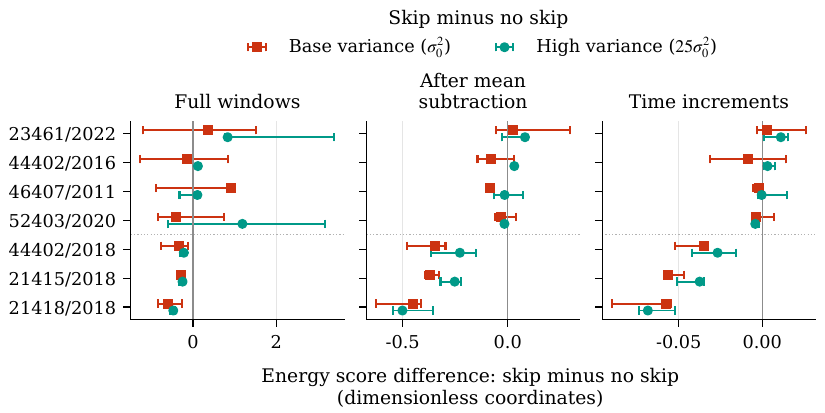}
\caption{DART energy score differences between skip and no skip at base variance $\sigma_0^2$ (squares) and high variance $25\sigma_0^2$ (circles). Negative values favor the skip. Scores use dimensionless coordinates. Points give median differences across three training seeds; bars span their full range. Scores average three streams per fit. Original records use network weights chosen by validation and equal window weights; 2018 records use final weights and UTC start hour weights. Table~\ref{tab:dart_contrasts} gives sampling variation.}
\label{fig:comparison}
\end{figure}

\subsection{DART case studies}
Energy score differences take both signs across the four original records. For 46407/2011 at base variance, subtracting each window's mean reverses the sign of the median difference between skip and no skip (Figure~\ref{fig:comparison}). For all three 2018 cases, the skip lowers all three mean energy scores at both observation variances. Within 21415/2018 and 21418/2018, energy score differences for full windows take both signs across intervals of 14 days and training seeds (Figure~\ref{fig:dart_intervals_additional}). PPCA has lower mean energy score than every neural fit on all seven cases and three scores, including at the same observation variance (Figure~\ref{fig:dart_ppca}). The leading 16 covariance eigenvectors account for almost all training variance, making PPCA a useful linear baseline (Appendix~\ref{sec:reference_distributions}). Its advantage combines differences in model class and fitting procedure.

\subsection{Estimating observation variance}
With encoder and decoder weights kept the same, estimating variance from training reconstruction error increases observation variance in all 42 fits trained at $\sigma_0^2$ and decreases it in all 42 fits trained at $25\sigma_0^2$. All three mean energy scores improve in the first group and worsen in the second (Figure~\ref{fig:dart_variance_refit}). After averaging streams within each fit, the absolute change in half the mean distance between model samples exceeds that in their mean distance to observations. Subtracting the former term determines the direction of the score change (Appendix~\ref{sec:variance_reset}). Minimizing the Gaussian reconstruction loss on training windows can worsen energy score on test windows.

\section{Conclusion}
For a Gaussian VAE with a constant input skip, exact rescaling can make reconstruction MSE arbitrarily small while preserving the generator. Separate DART fits give energy score differences of both signs, and PPCA has lower mean scores than every neural fit. Therefore, reconstruction error needs to be interpreted together with observation variance and generative scores.

\clearpage
\begin{ack}
This work extends research carried out at the Jet Propulsion Laboratory, California Institute of Technology, under a contract with the National Aeronautics and Space Administration. Parts of the development and verification were conducted with the Autonomous Research and Engineering System (ARES) agentic platform and the ARES CLI, both developed by the author. The author extends his sincere gratitude to Dr.\ Felix W.\ Landerer and Dr.\ Yuhe Tony Song from JPL for their guidance and support. The National Oceanic and Atmospheric Administration's National Data Buoy Center provided the DART observations, and the agency's National Centers for Environmental Information provided deployment records. The author has no competing interests to declare.
\end{ack}

\bibliographystyle{unsrtnat}
\bibliography{references}

\clearpage
\appendix
\raggedbottom
\setlength{\LTcapwidth}{\linewidth}
\section{Density, MSE, and energy score}
Consider the objective $\E\|x-f_\theta(z)-\alpha x\|^2+\beta_\alpha R$, where $R=\E_x\KL(q_\phi(z\mid x)\Vert p(z))$. Its reconstruction term sums squared residuals over observation coordinates, before division by $d$. Substituting $f_\theta=b g_\theta$ gives
\begin{equation}
\frac{\E\|x-f_\theta(z)-\alpha x\|^2+\beta_\alpha R}{b^2}
=\E\|x-g_\theta(z)\|^2+\frac{\beta_\alpha}{b^2}R.
\label{eq:scale}
\end{equation}
At $\alpha=0.8$, the reconstruction term is multiplied by $0.04$ and the relative KL coefficient becomes $\beta_\alpha/b^2$, equal to $25\beta_\alpha$. Keeping the KL coefficient constant increases its weight relative to squared error. Training uses the negative ELBO of the normalized density with KL coefficient one. To express this objective as a sum of squared residuals and a KL penalty, drop terms constant in the training parameters and multiply by $2\sigma_r^2$; the KL coefficient is then $2\sigma_r^2$. Averaging squared residuals over coordinates instead gives coefficient $2\sigma_r^2/d$. Equivalently, in negative ELBO expressed as a multiple of MSE plus KL, the MSE coefficient is $d/(2\sigma_r^2)$.

\begin{proof}[Proof of Proposition~\ref{prop:equivalence}]
The transformation $r=bx$ has Jacobian determinant $b^d$. After mapping, the conditional Gaussian for $r$ has mean $b\,g(z)$ and covariance $b^2\sigma_0^2I$. Multiplying this density by the Jacobian gives $\mathcal{N}(x;g(z),\sigma_0^2I)$. Combining it with the same prior preserves the distribution of observations and latent variables. Hence, the marginal and exact posterior also stay the same. The common encoder gives the same ELBO. For the same latent samples, each importance weight is also identical. Finally, $\|b x-b g(z)\|^2=b^2\|x-g(z)\|^2$ holds for every $z$, including the common encoder mean, and after averaging over the common encoder distribution.
\end{proof}

Reconstruction MSE evaluates the decoder at $z=\E_q Z$. For a nonlinear decoder, this value can lie above or below the MSE averaged over the encoder distribution, which enters the Gaussian reconstruction term of the ELBO. Both errors scale by $b^2$ under the mapping. Positive scaling preserves their rankings, including anomaly rankings defined by the chosen error. It also scales a threshold at the mean error plus a constant multiple of its standard deviation by the same factor, leaving classification decisions unchanged. These statements apply to two representations of one model. For models trained separately, assessing generation requires samples or densities in addition to reconstruction error.

Integrating Eq.~\ref{eq:density} over the latent prior gives the marginal density. Decoding the encoder mean evaluates the decoder at one latent value. Since the mapping leaves the marginal density unchanged, it preserves every strictly proper score of that density, including log score and, for distributions with finite first moments, energy score \citep{GneitingRaftery2007ProperScoringRules}. The finite importance estimate is also identical when the two representations use the same encoder distribution and the same latent samples.

The study concerns a constant scalar $\alpha<1$. The singular case $\alpha=1$, input dependent skips, and matrix coefficients each require separate analysis. The implicit generator samples noise once and solves for $X$. At $\alpha=0.8$ with residual variance one, its conditional variance is 25. An autoregressive process of order one, AR(1), repeatedly multiplies the previous state by 0.8 and adds constant forcing and fresh independent noise. With noise variance one, that process has stationary variance $1/(1-0.8^2)\approx2.778$.
Several VAE constructions illustrate how reconstruction and latent representation can diverge. Attention that passes information without sampling a latent variable can bypass the stochastic representation \citep{Bahuleyan2018VariationalAttention}; equal ELBO values can accompany different representations \citep{Alemi2018FixingBrokenELBO}. For linear conditional and hierarchical VAEs, analyses relate posterior collapse to correlation between inputs and outputs and to encoder variance learned during training \citep{Dang2024BeyondVanilla}. Dieng et al.\ connect $z$ to decoder layers to encourage dependence of $p(x\mid z)$ on the code \citep{Dieng2019SkipVAE}. An input skip instead passes $x$ into its reconstruction, providing a route that can ignore $z$. Generation still uses $z$ and noise, and Eq.~\ref{eq:density} gives its density for the constant skip considered here.

Under suitable assumptions on its regularization and corruption noise, an autoencoder can recover derivatives of the data log density from the reconstruction function it learns \citep{AlainBengio2014WhatRegularizedAutoEncodersLearn}. For the constant input skip, Eq.~\ref{eq:density} gives the generative density through a change of variables, whose Jacobian supplies the normalization.

For a scalar example, take $X\sim\mathcal{N}(0,1)$, $f=0$, and $\sigma_r^2=1$. Raising $\alpha$ from zero to 0.8 reduces expected reconstruction MSE from one to 0.04 and raises expected negative log likelihood from approximately 1.419 to 2.548 nats. Both quantities are population expectations; the likelihood calculation includes the density normalizer.

For each test window, $K$ latent samples from the encoder distribution give the estimate
\begin{equation}
\widehat\ell_K(x)=\log\left[\frac1K\sum_{k=1}^K
\frac{p_\alpha(x\mid z_k)p(z_k)}{q_\phi(z_k\mid x)}\right],
\qquad z_k\sim q_\phi(z\mid x).
\end{equation}
The expectation of this estimate is at most $\log p_\alpha(x)$, although a single realization can exceed it. At finite $K$, its value depends on both the encoder sampling distribution and the marginal density. The proper log score concerns exact marginal density. Log weights are calculated and accumulated in double precision after single precision neural calculations. Effective sample size, $(\sum w_k)^2/\sum w_k^2$, measures concentration of the importance weights. When the encoder poorly approximates the posterior, a few samples can carry almost all the importance weight. Comparing fits changes both the density being estimated and the distribution used to sample latent values.

For normalization $u=(h-m)/s$, with mean $m$ and scale $s$ estimated from training rows, log density in meter coordinates is $\log p_H(h)=\log p_U(u)-d\log s$. Log density comparisons use standardized DART coordinates and original Gaussian coordinates. All logarithms are natural, so log quantities are in nats. A density for the complete time series would also need a model of dependence between windows.

With observation mean $g(z)=f(z)/b$ and variance $v=\sigma_r^2/b^2$, negative ELBO per coordinate is
\begin{equation}
\frac{\mathcal L_\alpha(x)}{d}=\frac12\log(2\pi v)+\frac{\E_q\|x-g(Z)\|^2}{2dv}+\frac{\KL(q\Vert p)}d.
\label{eq:elbo_components}
\end{equation}
Multiplying $v$ by 25 increases the first term by $\log(25)/2$. The second term averages over the encoder distribution; evaluating only at the encoder mean would define a different objective. In expectation, the IW lower bound with $K$ samples is at least as tight as the ELBO. The improvement can differ between models. For the same set of log weights, Jensen's inequality gives $\widehat\ell_K-K^{-1}\sum_k\log w_k\geq0$: the log of the mean importance weight is at least the mean log weight. This compares two functions of the same importance weights. Error relative to exact marginal log density remains unknown; replacing sampled log ratios with analytic KL defines a separate comparison.

For independent samples $y_1,\ldots,y_M$ from the model's marginal distribution, the Gneiting and Raftery energy score has the estimator
\begin{equation}
\widehat{\mathrm{ES}}(x)=\frac1M\sum_i\|y_i-x\|_2
-\frac{1}{M(M-1)}\sum_{i<j}\|y_i-y_j\|_2.
\end{equation}
The second term uses distinct sample indices. If samples are identically distributed and independent of each other and the evaluation observation, the estimate is unbiased for population energy loss. That loss is strictly proper for distributions with finite first moments \citep{GneitingRaftery2007ProperScoringRules}: its expectation is uniquely minimized by the true distribution. Finite ensembles introduce Monte Carlo variability. All conditions use the same sample indices and evaluation seeds.

The Gaussian simulator has 184 covariance eigenvalues equal to $0.0625$ and eight equal to $s_j^2+0.0625$, where $s_j$ decreases linearly from 4 to 1. The original PPCA estimate is $\sigma_0^2\approx0.0615716325$, so the conditions at high variance use $1.5392908133$. Among the eight larger eigenvalues, only the smallest, $1.0625$, lies below this noise floor. The 185 eigenvalues below the floor rule out recovery of the full covariance by any $g(Z)+\epsilon$ with independent $\epsilon\sim\mathcal N(0,\sigma_x^2I)$. This calculation applies to nonlinear decoders as well as linear ones. The covariance restriction identifies excess variance, while its contribution to score differences and the optimality of each fit remain uncertain. The numerical equality checks below use observation variances 0.0625 and 1.5625.

\section{Adam under decoder rescaling}
Let $w$ contain the output kernel and bias for the observation mean, and write the corresponding parameters of the residual decoder as $r=bw$. Let $h_t$ be the gradient at step $t$. At corresponding parameter values, the same minibatches and latent samples give $h_t^r=h_t^w/b$, provided clipping and weight decay are inactive. If this relation holds throughout the preceding steps and the moments start at zero, Adam's averages of gradients and squared gradients, $m_t$ and $s_t$, satisfy $m_t^r=m_t^w/b$ and $s_t^r=s_t^w/b^2$. In TensorFlow 2.15, epsilon is added to the square root of the second moment before bias correction; the correction appears in the step size. For learning rate $\eta_r$, stabilizer $\epsilon_r$, and moment decay coefficients $\beta_1,\beta_2$, an update to the residual parameters gives the following step in $w$:
\begin{equation}
\Delta w_t=-\frac{\eta_r}{b}\frac{\sqrt{1-\beta_2^t}}{1-\beta_1^t}
\frac{m_t^w}{\sqrt{s_t^w}+b\epsilon_r}.
\label{eq:adam_transport}
\end{equation}
To match a step taken directly in $w$ with learning rate $\eta_w$ and stabilizer $\epsilon_w$, exact arithmetic requires $\eta_r=b\eta_w$ and $\epsilon_r=\epsilon_w/b$. If epsilon is negligible, scaling the learning rate to $\eta_r=b\eta_w$ while leaving epsilon unchanged approximates that step. Under this approximation, a constant learning rate for the residual decoder corresponds to a rate five times larger in $w$ when $b=0.2$. Clipping the norm of all gradients after rescaling the output parameters can also change updates elsewhere in the network. These relations follow from parameter rescaling and Adam's rule \citep{Dinh2017SharpMinima,KingmaBa2015Adam}. Their scope is individual steps under the assumptions on gradient history above. Training speed and generative performance of the nonlinear fits require separate evidence.

A numerical check trains pairs of models using parameters for observation means in both, with the input skip used only to calculate reconstruction. Four Gaussian runs each take 10,000 Adam steps and cover both observation variances. With identical implementations and random samples, all 48 parameter and optimizer state tensors agree after every step, including steps with active clipping. Repeating the computations reproduces the saved states and final weights exactly. This verifies equality for the common parameterization and random inputs used in the check. The DART models train separately using parameters of the residual decoder, so the check leaves the cause of their score differences open.

A small encoder KL can occur away from a likelihood optimum. The Gaussian population's largest covariance eigenvalue is $\lambda_1=16.0625$, and seven eigenvalues exceed $v=1.5625$. In a linear Gaussian model with exact inference, changing zero loading to $W=t u_1$, where $u_1$ is a leading covariance eigenvector, changes the population negative log likelihood (NLL) by
\begin{equation}
\Delta\mathrm{NLL}=\tfrac12\left[\log(1+t^2/v)-\frac{\lambda_1t^2}{v(v+t^2)}\right]
=-2.9696\,t^2+O(t^4).
\end{equation}
Under these linear assumptions, moving away from zero loading reduces population negative log likelihood, so zero loading fails the local minimum condition. Lucas et al.\ analyze collapse in linear VAEs with a shared diagonal encoder covariance \citep{Lucas2019LinearVAE}; He et al.\ study inference networks that lag behind the changing model posterior \citep{He2019LaggingInference}. Whether the nonlinear fits here have reached a local optimum remains unresolved.

\section{Windows, splits, and training}
\label{sec:data}
The synthetic basis $U$ alternates cosine and sine columns, $\cos(2\pi kt/192)$ and $\sin(2\pi kt/192)$, for harmonics $k=1,\ldots,4$ at $t=0,\ldots,191$. Scaling by $\sqrt{2/192}$ gives $U^\top U=I_8$. The loading matrix is $L=U\operatorname{diag}(s)$, with eight singular values equally spaced from 4 to 1. Each realization is $X=LZ+0.25\epsilon$, with independent $Z\sim\mathcal{N}(0,I_8)$ and $\epsilon\sim\mathcal{N}(0,I_{192})$. NumPy PCG64 streams 1101, 1102, and 1103 supply training, validation, and test samples, respectively. In each partition, generation samples all latent factors before sampling observation noise.

A separate coordinate check uses a mixture of two Gaussians whose overall covariance equals that of the original Gaussian: replace the first loading 4 by 2 and add an independent random sign times $\sqrt{12}$ along the first basis vector, using random number generator streams 2101 to 2103. The 60 original fits use the Gaussian family and four DART records; the 36 additional fits use the three DART records from 2018. Gaussian observations retain the simulator's dimensionless coordinates, without mean subtraction or rescaling to unit variance.

For the four original DART records, split dates are midnight UTC on April 28 and November 25 for 23461; August 9 and October 21 for 44402; August 12 and October 23 for 46407; and September 21 and October 30 for 52403, in their respective years. For dates $s_1,s_2$, training uses $t<s_1-2$ days, validation uses $s_1+2\text{ days}\leq t<s_2-2$ days, and test uses $t\geq s_2+2$ days. Excluding $[s-2\text{ days},s+2\text{ days})$ around each split keeps observations in separate partitions, although serial dependence remains. Windows begin at each segment's first observation and advance by 16 observations for training and 192 for validation and test. In these four records, values above 9000 indicate missing data. Other eligible extremes remain. Duplicate observed T=1 timestamps are excluded, and a segment ends at an invalid value, another reporting type, a break in the order of source rows, or a time interval other than 900 seconds.

\begin{figure}[tp]
\centering
\includegraphics[width=\linewidth,height=0.81\textheight,keepaspectratio]{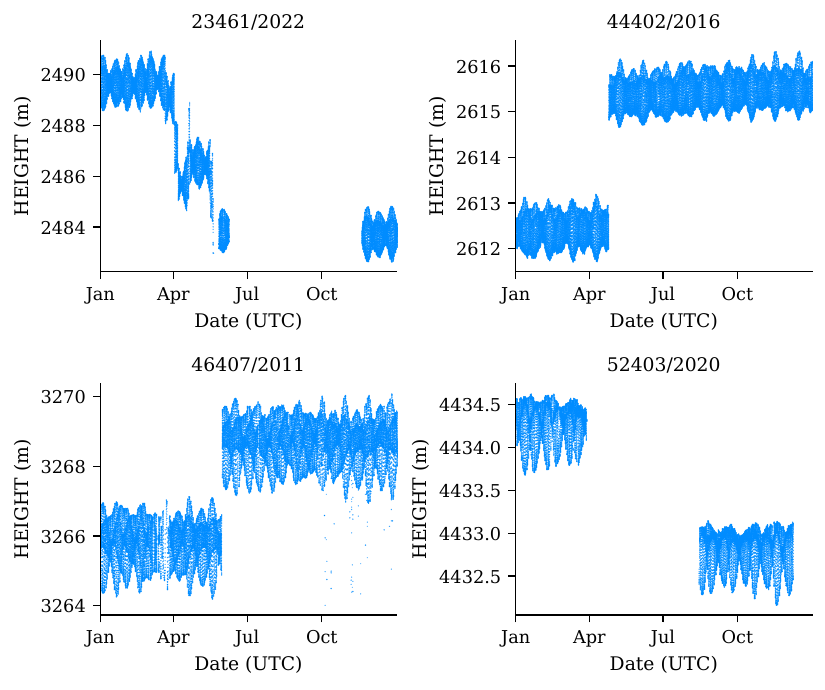}
\caption{T=1 HEIGHT in meters for the four original records, identified by station and year. Points are NOAA HEIGHT observations; blank intervals contain no usable T=1 observations. The steps at 44402 and 46407 occur near the National Centers for Environmental Information (NCEI) deployment dates cited in Appendix~\ref{sec:scope}. At 52403, the levels differ before and after an interval without usable T=1 observations from March to August that includes the catalogue deployment date of 13 August 2020. The time of the level change remains uncertain; the cause of the 23461 spikes is also unresolved. Figures~\ref{fig:dart_ecdf_fresh} and~\ref{fig:dart_spectra_fresh} show feature distributions and spectra for the three 2018 records.}
\label{fig:context}
\end{figure}

The encoder has hidden widths 128/64 with rectified linear unit (ReLU) activations and separate outputs for Gaussian means and log variances. The decoder has hidden widths 64/128 with ReLU activations and a linear residual output. Log variance is restricted to $[-12,12]$, limiting the variational family. Training uses no input clipping, moving baseline, output smoothing, L2 penalty, or KL warmup. Adam uses batches of 64, learning rate $10^{-3}$, at most 10,000 steps, and a global gradient norm cap of 10. Training samples windows uniformly with replacement for the original records and with UTC start hour weights for 2018. Seeds are 42, 1729, and 31415, in the order shown in plots of individual fits. Conditions with the same seed begin with the same weights except that the residual output kernel and bias are multiplied by $b$ to preserve observation means. Comparisons pair fits with the same training seed.

For the original DART records, normalization counts each eligible training row once, including rows in fragments too short to form a window. The standard deviation uses divisor $n$, the number of these rows. PPCA estimates residual variance by averaging covariance eigenvalues outside its loading directions, subject to a positive minimum. For the skip at base variance, this minimum is applied before multiplying the base variance by $b^2$.

Original validation uses 64 importance samples per window and random number generator stream 41001. Validation runs every 250 Adam steps; training stops after eight consecutive checks without improvement and retains the weights with the best validation estimate. Primary evaluation uses stream 42001. A sensitivity check with 1,024 samples uses 42002, and a separate repeat with 256 samples uses 42003. Prior samples use stream 43001. Final network weights are retained separately from the weights chosen by validation. All four conditions use this stopping rule; the iteration chosen by validation and the total number of steps can differ between fits.

\subsection{Additional DART records}
The additional cases use $[2018\text{-}01\text{-}01,2019\text{-}01\text{-}01)$, with splits 219 and 292 days after the start and two days excluded on each side of each split. Windows start at 00:00, 04:00, 08:00, 12:00, 16:00, or 20:00 UTC. Each requires 192 consecutive source rows with finite T=1 HEIGHT values at most 9000 and unambiguous timestamps 900 s apart; any interruption ends a usable sequence. Eligibility requires an interval of at least 240 days and valid routine observations in at least 80\% of all UTC slots at 15 minute spacing across that complete interval. Every partition must include windows from all six start hours. It must also contain a subset of windows that share no source observations: at least 60 training windows, 20 validation windows, and 20 test windows. These subset counts are 101/33/33 for 44402, 104/32/23 for 21415, and 87/32/22 for 21418, in training/validation/test order. Model fitting and evaluation use the larger overlapping sets in Table~\ref{tab:data}. The dates determine the splits; the subset counts measure the available data.

\begin{table}[tp]
\centering
\setlength{\tabcolsep}{3pt}
\caption{Window counts and normalization constants, with routine T=1 DART records identified by station and year. Mean and scale are the training mean and population standard deviation in meters. Normalization counts each eligible training row once for the original DART records; for 2018, it counts each unique training row in an eligible window once. Gaussian data keep their original coordinates, shown as mean zero and scale one. Base variance $\sigma_0^2$ is expressed in squared model coordinates. For 2018, PPCA validation chooses this variance before neural fitting.}
\label{tab:data}
\begin{tabular}{lrrrrrr}
\toprule
Case & Train & Validation & Test & Mean & Scale & $\sigma_0^2$\\
\midrule
Gaussian & 8192 & 2048 & 2048 & 0 & 1 & 0.06157 \\
23461/2022 & 544 & 13 & 17 & 2488.8962 & 1.481 & 0.0001 \\
44402/2016 & 1153 & 34 & 35 & 2613.8301 & 1.518 & 0.0001 \\
46407/2011 & 983 & 26 & 27 & 3266.7902 & 1.48 & 0.0001 \\
52403/2020 & 339 & 17 & 16 & 4433.8446 & 0.7036 & 0.0001
\\
\midrule
44402/2018 & 1150 & 390 & 369 & 2615.7276 & 0.33 & 0.0001 \\
21415/2018 & 1227 & 364 & 215 & 4798.7832 & 0.3153 & 0.0001 \\
21418/2018 & 895 & 378 & 189 & 5726.9972 & 0.2634 & 0.0001
\\
\bottomrule
\end{tabular}
\end{table}

Within each partition, a window starting in UTC hour $j$ receives weight $1/(6n_j)$, where $n_j$ is the number of eligible starts in that hour. Each hour contributes one sixth of the weight to neural sampling, baseline fitting, variance choice, validation, and evaluation. Normalization gives each unique training observation in an eligible window equal weight and uses a population standard deviation. The observation set differs from that used for the original records, and equal observation weights differ from the UTC start hour weights on training windows. A secondary evaluation assigns equal weight to windows using the same scores. The data contain the eligible source measurements, without interpolation or augmentation. Equal weights for UTC start hours balance the clock times; tidal phases also depend on date and constituent frequency.

The 44402/2018 record was chosen after examination of the station's 2016 record. A subsequent search for additional sites excluded the four original stations and required catalogue coverage for all of 2018. Stations were examined in identifier order; 21415 and 21418 were the first two annual records to meet the data requirements above. Both were included before their neural model results were available, and all 36 fits remain in the analysis. The catalogues give deployment periods \citep{NCEIDART44402,NCEIDART21415,NCEIDART21418}. Hardware continuity within these deployment periods remains uncertain.

The skip lowers mean energy scores on all three additional records in each of two checks: evaluation at network weights chosen by validation, and evaluation with equal weight for every window. The original and additional records also differ in calendar splits, normalization data, evaluation stride, training duration, and selection of network weights. Attributing their score differences to deployment history requires accounting for these changes.

For each additional record, PPCA uses at most 16 loading directions, the training window weights, and variance values $10^{-4},5\times10^{-4},2.5\times10^{-3},1.25\times10^{-2},6.25\times10^{-2}$. The variance with the highest exact validation log density, averaged with the validation window weights, is chosen; an exact tie favors the larger variance. All three records select the smallest value, $10^{-4}$, before neural fitting. Test results play no part in the grid choice. Each neural fit then takes 10,000 Adam steps. Every 1,000 steps, energy score on full validation windows is evaluated with 512 prior samples and the validation window weights. The network weights with the lowest score are saved for a secondary evaluation, with ties favoring the earlier step. Within a seed, conditions share training hyperparameters and initial observation means. Adam and gradient clipping act on the parameters of the residual decoder, allowing the training paths to differ.

\section{ELBO terms and Gaussian baselines}
Each test window receives the preceding window's encoder mean, and the first window receives the last window's encoder mean. The change in MSE in observation coordinates measures sensitivity to this reassignment.

For the Gaussian comparison from no skip at base variance to skip at high variance, mean conditional negative log likelihood, averaged over encoder samples, increases from 0.04692 to 1.17735 nats per coordinate. The analytic KL decreases from 0.09338 to 0.02497 nats per coordinate, by less than the NLL increase. Without a skip at high variance, estimated negative IW log density ranges from 1.202 to 1.252 nats per coordinate. Encoder KL measures divergence from the prior, while reassignment measures sensitivity to which encoder mean a window receives. One fit without a skip at high variance has encoder KL of about $0.0003$ nats per window, compared with 4.81 for the skip fit at the same variance. Reassigning encoder means changes MSE in observation coordinates by zero at the recorded precision for that fit without a skip and by 0.405 and 0.453 for the other two fits without a skip at high variance.

For 46407 in 2011, moving from no skip at base variance to skip at high variance lowers mean NLL from 27.0860 to $-0.6335$ and raises mean MSE from 0.006154 to 0.007217. Both averages use the encoder distribution; NLL is in nats per coordinate and MSE is in squared dimensionless coordinates. Equation~\ref{eq:elbo_components} accounts for the different directions through the density normalizer and the division of error by variance. When reconstruction instead decodes the encoder mean, the skip at high variance also has higher MSE in observation coordinates in two of the three comparisons with the same training seed.

The component analysis uses all 60 original models at network weights chosen by validation. The DART extension adds results for the 36 final models from 2018 to the 48 original DART models and keeps the Gaussian analysis. For each test window, three independent streams each provide 256 samples from the encoder. The same samples determine conditional log likelihood, $\log q-\log p(z)$, and importance weights. Encoder KL to the prior is calculated analytically. Table~\ref{tab:components} averages first over test windows with each record's weights, then over three sampling streams, then over three training seeds. Reconstruction from the encoder mean and analytic KL depend only on the inputs and network weights. Negative ELBO uses analytic KL; the Jensen difference uses the log ratios from the actual samples. Exact marginal log densities are available for the Gaussian baselines, but remain unknown for the neural fits.

\clearpage
\begin{table}[H]
\centering
\setlength{\tabcolsep}{1.8pt}
\caption{Reconstruction errors and ELBO terms at the observation variance used for training. The original records use network weights chosen by validation and equal window weights; the 2018 records use final network weights and UTC start hour weights. Raw MSE includes the input skip; division by $b^2$ expresses the error in observation coordinates after decoding the encoder mean. For 2018, raw MSE is calculated as $b^2$ times this error; original fits use the raw MSE saved during evaluation. Expected MSE and conditional negative log likelihood (NLL) average over samples from $q_\phi(z\mid x)$. MSE is in squared model coordinates; log quantities are in nats per coordinate. Each entry averages over test windows with the weights above, then over three sampling streams, then over three fits. Before rounding, NLL plus analytic KL equals negative ELBO. IW estimates negative marginal log density using finite $K$. Base and high denote observation variances $\sigma_0^2$ and $25\sigma_0^2$. The Skip column uses $\checkmark$ for $\alpha=0.8$ and $\times$ for $\alpha=0$.}
\label{tab:components}
\small
\begin{tabular}{lllrrrrrrr}
\toprule
Case & Skip & Variance & \shortstack{Raw\\MSE} & \shortstack{MSE\\$ / b^2$} & \shortstack{Expected\\MSE} & NLL/$d$ & KL/$d$ & \shortstack{Negative\\ELBO/$d$} & \shortstack{Negative\\IW/$d$}\\
\midrule
Gaussian & $\times$ & Base & 0.0616 & 0.0616 & 0.06425 & 0.04692 & 0.09338 & 0.1403 & 0.1332 \\
Gaussian & $\checkmark$ & Base & 0.002496 & 0.06239 & 0.06498 & 0.05287 & 0.09343 & 0.1463 & 0.1373 \\
Gaussian & $\times$ & High & 0.1985 & 0.1985 & 0.2182 & 1.205 & 0.01483 & 1.22 & 1.22 \\
Gaussian & $\checkmark$ & High & 0.003683 & 0.09209 & 0.1316 & 1.177 & 0.02497 & 1.202 & 1.201 \\
23461/2022 & $\times$ & Base & 0.0008939 & 0.0008939 & 0.0008945 & 0.7864 & 1.447 & 2.233 & 2.065 \\
23461/2022 & $\checkmark$ & Base & 0.000102 & 0.002551 & 0.002552 & 9.073 & 1.372 & 10.44 & 10.06 \\
23461/2022 & $\times$ & High & 0.001201 & 0.001201 & 0.001207 & -1.835 & 0.6391 & -1.196 & -1.234 \\
23461/2022 & $\checkmark$ & High & \ensuremath{8.379\times10^{-5}} & 0.002095 & 0.002121 & -1.653 & 0.5324 & -1.12 & -1.228 \\
44402/2016 & $\times$ & Base & 0.0003423 & 0.0003423 & 0.0003479 & -1.947 & 0.5698 & -1.377 & -1.573 \\
44402/2016 & $\checkmark$ & Base & \ensuremath{1.509\times10^{-5}} & 0.0003771 & 0.0003828 & -1.772 & 0.5071 & -1.265 & -1.58 \\
44402/2016 & $\times$ & High & 0.0003269 & 0.0003269 & 0.0003903 & -1.999 & 0.1313 & -1.867 & -1.89 \\
44402/2016 & $\checkmark$ & High & \ensuremath{1.723\times10^{-5}} & 0.0004308 & 0.0004851 & -1.98 & 0.1266 & -1.853 & -1.894 \\
46407/2011 & $\times$ & Base & 0.006143 & 0.006143 & 0.006154 & 27.09 & 0.553 & 27.64 & 27 \\
46407/2011 & $\checkmark$ & Base & 0.000274 & 0.00685 & 0.006867 & 30.65 & 0.5804 & 31.23 & 30.12 \\
46407/2011 & $\times$ & High & 0.006611 & 0.006611 & 0.006824 & -0.712 & 0.2272 & -0.4848 & -0.6253 \\
46407/2011 & $\checkmark$ & High & 0.0002783 & 0.006957 & 0.007217 & -0.6335 & 0.2009 & -0.4326 & -0.6198 \\
52403/2020 & $\times$ & Base & 0.0009731 & 0.0009731 & 0.0009804 & 1.216 & 0.5951 & 1.811 & 1.38 \\
52403/2020 & $\checkmark$ & Base & \ensuremath{7.928\times10^{-5}} & 0.001982 & 0.00199 & 6.265 & 0.5363 & 6.801 & 5.821 \\
52403/2020 & $\times$ & High & 0.001278 & 0.001278 & 0.001487 & -1.779 & 0.1976 & -1.582 & -1.686 \\
52403/2020 & $\checkmark$ & High & \ensuremath{9.723\times10^{-5}} & 0.002431 & 0.002613 & -1.554 & 0.2245 & -1.33 & -1.533 \\
44402/2018 & $\times$ & Base & 0.001983 & 0.001983 & 0.001987 & 6.247 & 0.5841 & 6.831 & 6.402 \\
44402/2018 & $\checkmark$ & Base & 0.0001096 & 0.00274 & 0.002743 & 10.03 & 0.5617 & 10.59 & 10.05 \\
44402/2018 & $\times$ & High & 0.002506 & 0.002506 & 0.002634 & -1.55 & 0.2993 & -1.251 & -1.383 \\
44402/2018 & $\checkmark$ & High & 0.0001297 & 0.003242 & 0.003332 & -1.41 & 0.3115 & -1.099 & -1.253 \\
21415/2018 & $\times$ & Base & 0.001018 & 0.001018 & 0.001021 & 1.421 & 0.6215 & 2.042 & 1.753 \\
21415/2018 & $\checkmark$ & Base & \ensuremath{9.305\times10^{-5}} & 0.002326 & 0.002329 & 7.961 & 0.5846 & 8.546 & 8.018 \\
21415/2018 & $\times$ & High & 0.001274 & 0.001274 & 0.001435 & -1.79 & 0.2931 & -1.497 & -1.61 \\
21415/2018 & $\checkmark$ & High & 0.0001031 & 0.002576 & 0.002735 & -1.53 & 0.286 & -1.244 & -1.418 \\
21418/2018 & $\times$ & Base & 0.001534 & 0.001534 & 0.001544 & 4.031 & 0.6301 & 4.661 & 4.16 \\
21418/2018 & $\checkmark$ & Base & 0.0001236 & 0.00309 & 0.0031 & 11.81 & 0.5499 & 12.36 & 11.33 \\
21418/2018 & $\times$ & High & 0.001673 & 0.001673 & 0.001863 & -1.704 & 0.2886 & -1.416 & -1.54 \\
21418/2018 & $\checkmark$ & High & 0.0001342 & 0.003355 & 0.003529 & -1.371 & 0.2885 & -1.082 & -1.292
\\
\bottomrule
\end{tabular}
\end{table}

\begin{table}[H]
\centering
\caption{Gaussian baseline scores on test realizations; lower values are better. Negative log densities are exact and use nats per coordinate. Energy score uses the original dimensionless Gaussian coordinates and 512 samples per baseline, obtained by applying the covariance Cholesky factor to standard normal samples. True Gaussian denotes the known generating density. Figure~\ref{fig:dart_ppca} uses a separate evaluation of DART baselines with three sampling streams.}
\begin{tabular}{llrr}
\toprule
Case & Model & Negative log score & Energy score\\
\midrule
Gaussian & True Gaussian & 0.1255 & 5.717 \\
Gaussian & PPCA & 0.1266 & 5.72 \\
Gaussian & Diagonal Gaussian & 0.9087 & 5.793 \\
\\
\bottomrule
\end{tabular}
\end{table}

Proposition~\ref{prop:equivalence} is an identity in exact arithmetic. Applying its map in single precision to all 60 original fits gives maximum absolute differences of 1 nat in conditional log density and \ensuremath{3.576\times10^{-7}} in conditional means, in model coordinates. The 1 nat difference occurs when mapping a DART model from skip to no skip at base variance, where the absolute log density is large; the comparison uses relative tolerance $3\times10^{-5}$ and absolute tolerance 0.01 nats. Gaussian differences remain below $2\times10^{-4}$ nats. This check uses the first eight validation windows and latent samples from the prior. Marginal IW estimates use encoder samples on the full test set, and their accuracy remains uncertain after this check of conditional density. A separate sampling comparison on the initial test subset finds a largest absolute change of 0.2616 nats per coordinate in a fit's mean IW estimate between 256 and 1,024 samples.

\begingroup
\setlength{\tabcolsep}{3pt}
\small
\begin{longtable}{lllrrrrrrrr}
\caption{All 12 original Gaussian fits. Raw MSE decodes the encoder mean and includes the input skip; MSE divided by $b^2$ expresses the error in observation coordinates. Both are in squared dimensionless coordinates. Negative IW estimates negative log density in nats per coordinate; KL measures encoder divergence in nats per realization. Energy score uses 512 prior samples per fit. Errors, KL, and energy score are averaged over all test realizations for each fit. Lower errors and scores are better. ESS is computed from 256 importance samples per realization and then summarized by its median across realizations. Step is the Adam iteration chosen by validation. Base and high denote observation variances $\sigma_0^2$ and $25\sigma_0^2$. The Skip column uses $\checkmark$ for $\alpha=0.8$ and $\times$ for $\alpha=0$.}\\
\toprule
Case & Skip & Variance & Seed & Step & \shortstack{Raw\\MSE} & MSE/$b^2$ & \shortstack{Negative\\IW} & KL & \shortstack{Energy\\score} & ESS\\
\midrule
\endfirsthead
\toprule
Case & Skip & Variance & Seed & Step & \shortstack{Raw\\MSE} & MSE/$b^2$ & \shortstack{Negative\\IW} & KL & \shortstack{Energy\\score} & ESS\\
\midrule
\endhead
Gaussian & $\times$ & Base & 42 & 5000 & 0.06162 & 0.06162 & 0.1335 & 17.94 & 5.723 & 41.3 \\
Gaussian & $\times$ & Base & 1729 & 5250 & 0.0617 & 0.0617 & 0.1334 & 18.06 & 5.724 & 35.68 \\
Gaussian & $\times$ & Base & 31415 & 7750 & 0.06149 & 0.06149 & 0.1326 & 17.79 & 5.721 & 52.9 \\
Gaussian & $\checkmark$ & Base & 42 & 10000 & 0.002494 & 0.06236 & 0.1371 & 17.9 & 5.719 & 31.33 \\
Gaussian & $\checkmark$ & Base & 1729 & 6000 & 0.002497 & 0.06243 & 0.1376 & 17.96 & 5.722 & 29.33 \\
Gaussian & $\checkmark$ & Base & 31415 & 7500 & 0.002496 & 0.06239 & 0.1371 & 17.96 & 5.724 & 29.24 \\
Gaussian & $\times$ & High & 42 & 750 & 0.3608 & 0.3608 & 1.252 & 0.0003156 & 6.901 & 255.8 \\
Gaussian & $\times$ & High & 1729 & 5500 & 0.1306 & 0.1306 & 1.205 & 3.881 & 7.088 & 209.9 \\
Gaussian & $\times$ & High & 31415 & 6250 & 0.1042 & 0.1042 & 1.202 & 4.658 & 7.118 & 206.5 \\
Gaussian & $\checkmark$ & High & 42 & 10000 & 0.003662 & 0.09155 & 1.201 & 4.814 & 7.13 & 202.3 \\
Gaussian & $\checkmark$ & High & 1729 & 5750 & 0.003704 & 0.09261 & 1.201 & 4.78 & 7.123 & 180.1 \\
Gaussian & $\checkmark$ & High & 31415 & 8750 & 0.003684 & 0.0921 & 1.201 & 4.787 & 7.118 & 196.3 \\
\\
\bottomrule
\end{longtable}
\endgroup

\subsection{Importance sampling estimates of log density}
\begin{table}[H]
\centering
\caption{Differences in negative IW log density estimates at $K=256$, in nats per coordinate, for the original records. Fits are paired by training seed. Each cell gives the median difference across three seeds and the full range below it; negative values favor the first condition. DART labels identify station and year. The ranking by exact marginal likelihood remains uncertain with these finite sample estimates. The first two comparisons keep variance the same; the last changes both skip and variance.}
\label{tab:results}
\begin{tabular}{lccc}
\toprule
Case & \shortstack{Skip vs no skip\\Base variance} & \shortstack{Skip vs no skip\\High variance} & \shortstack{Skip, high variance\\vs no skip, base variance}\\
\midrule
Gaussian & \shortstack{$0.004143$\\$[0.003679,\,0.004428]$} & \shortstack{$-0.004057$\\$[-0.05053,\,-0.0009969]$} & \shortstack{$1.068$\\$[1.068,\,1.069]$} \\
23461/2022 & \shortstack{$9.496$\\$[3.226,\,11.24]$} & \shortstack{$-0.02887$\\$[-0.2702,\,0.3098]$} & \shortstack{$-4.328$\\$[-4.52,\,-1.029]$} \\
44402/2016 & \shortstack{$-0.1457$\\$[-0.7499,\,0.8596]$} & \shortstack{$-0.002817$\\$[-0.007894,\,0.00181]$} & \shortstack{$-0.3495$\\$[-0.4308,\,-0.1869]$} \\
46407/2011 & \shortstack{$2.374$\\$[2.168,\,4.924]$} & \shortstack{$-0.01531$\\$[-0.02908,\,0.0532]$} & \shortstack{$-26.65$\\$[-30.07,\,-26.16]$} \\
52403/2020 & \shortstack{$4.26$\\$[3.937,\,5.114]$} & \shortstack{$0.1918$\\$[0.05503,\,0.2004]$} & \shortstack{$-2.595$\\$[-3.644,\,-2.527]$}
\\
\bottomrule
\end{tabular}
\end{table}

A sensitivity check follows the low initial DART effective sample sizes by evaluating the same network weights with 4,096 and 16,384 samples in two additional sampling streams. Each sample count uses the first 32 test windows, or all windows when fewer are available: 32 Gaussian, 17 from 23461, 32 from 44402, 27 from 46407, and 16 from 52403. Table~\ref{tab:refinement} compares these windows with their original estimates using 256 samples. For Gaussian data and 44402, this subset is smaller than the full test set used in the main results.

Even with more samples, the accuracy of the marginal log density estimate remains unknown. Between 256 and 16,384 samples, the largest absolute change in a fit's mean estimate was 0.468 nats per coordinate on DART and 0.000868 on Gaussian data. At 16,384 samples, median DART ESS per fit ranged from 1.0001 to 23.4. At base variance, the skip model's estimated negative IW log density exceeds that of the model without a skip for 23461, 46407, and 52403 at every sample count and encoder distribution tested. The direction for 44402 varies by training seed. At high variance, differences between skip and no skip take both signs for 23461, 44402, and 46407; on 52403, the skip model's estimate exceeds that of the model without a skip at every seed and sample count tested.

\begin{table}[H]
\centering
\setlength{\tabcolsep}{2pt}
\caption{Importance estimates with more samples on the same initial test windows. DART labels give station and year. Differences compare fits that share a training seed, at the observation variances shown. Each difference subtracts the negative IW log density without a skip from the value with a skip. Cells give the median difference in nats per coordinate and the full range across three seeds underneath. Negative values indicate a lower estimated score for the skip model. The final row for each case changes both skip use and observation variance. Base and high denote $\sigma_0^2$ and $25\sigma_0^2$. The skip coefficient is 0.8 when present and zero otherwise. The Gaussian and 44402 subsets are smaller than their full test sets.}
\label{tab:refinement}
\small
\renewcommand{\arraystretch}{1.8}
\begin{tabular}{lccccc}
\toprule
Case & \multicolumn{2}{c}{\shortstack{Observation\\variance}} & $K=256$ & $K=4096$ & $K=16384$\\
\cmidrule(lr){2-3}
& Skip & No skip & & &\\
\midrule
Gaussian & Base & Base & \shortstack{$0.003605$\\$[0.001233,\,0.004018]$} & \shortstack{$0.003406$\\$[0.001593,\,0.00391]$} & \shortstack{$0.003707$\\$[0.001884,\,0.00389]$} \\
Gaussian & High & High & \shortstack{$-0.00525$\\$[-0.05736,\,\ensuremath{9.69\times10^{-5}}]$} & \shortstack{$-0.005248$\\$[-0.05733,\,\ensuremath{9.666\times10^{-5}}]$} & \shortstack{$-0.005238$\\$[-0.05736,\,0.000105]$} \\
Gaussian & High & Base & \shortstack{$1.068$\\$[1.068,\,1.069]$} & \shortstack{$1.069$\\$[1.069,\,1.069]$} & \shortstack{$1.069$\\$[1.069,\,1.069]$} \\
23461/2022 & Base & Base & \shortstack{$9.496$\\$[3.226,\,11.24]$} & \shortstack{$9.462$\\$[3.209,\,11.11]$} & \shortstack{$9.455$\\$[3.185,\,11.05]$} \\
23461/2022 & High & High & \shortstack{$-0.02887$\\$[-0.2702,\,0.3098]$} & \shortstack{$-0.03072$\\$[-0.3078,\,0.302]$} & \shortstack{$-0.02819$\\$[-0.3055,\,0.3093]$} \\
23461/2022 & High & Base & \shortstack{$-4.328$\\$[-4.52,\,-1.029]$} & \shortstack{$-4.286$\\$[-4.527,\,-1.026]$} & \shortstack{$-4.262$\\$[-4.502,\,-1.011]$} \\
44402/2016 & Base & Base & \shortstack{$-0.1264$\\$[-0.7203,\,0.8996]$} & \shortstack{$-0.1789$\\$[-0.7328,\,0.8346]$} & \shortstack{$-0.1953$\\$[-0.7254,\,0.807]$} \\
44402/2016 & High & High & \shortstack{$-0.00224$\\$[-0.009321,\,0.002505]$} & \shortstack{$-0.01095$\\$[-0.01247,\,0.002069]$} & \shortstack{$-0.01189$\\$[-0.01382,\,0.0006799]$} \\
44402/2016 & High & Base & \shortstack{$-0.2881$\\$[-0.3638,\,-0.1447]$} & \shortstack{$-0.2255$\\$[-0.3328,\,-0.1409]$} & \shortstack{$-0.1959$\\$[-0.3169,\,-0.1381]$} \\
46407/2011 & Base & Base & \shortstack{$2.374$\\$[2.168,\,4.924]$} & \shortstack{$2.281$\\$[2.09,\,4.798]$} & \shortstack{$2.275$\\$[2.011,\,4.8]$} \\
46407/2011 & High & High & \shortstack{$-0.01531$\\$[-0.02908,\,0.0532]$} & \shortstack{$-0.02418$\\$[-0.03144,\,0.04631]$} & \shortstack{$-0.03034$\\$[-0.03092,\,0.04173]$} \\
46407/2011 & High & Base & \shortstack{$-26.65$\\$[-30.07,\,-26.16]$} & \shortstack{$-26.54$\\$[-29.86,\,-26.04]$} & \shortstack{$-26.48$\\$[-29.81,\,-25.97]$} \\
52403/2020 & Base & Base & \shortstack{$4.26$\\$[3.937,\,5.114]$} & \shortstack{$4.112$\\$[3.735,\,4.988]$} & \shortstack{$4.139$\\$[3.595,\,4.974]$} \\
52403/2020 & High & High & \shortstack{$0.1918$\\$[0.05503,\,0.2004]$} & \shortstack{$0.168$\\$[0.031,\,0.1719]$} & \shortstack{$0.1566$\\$[0.0255,\,0.1715]$} \\
52403/2020 & High & Base & \shortstack{$-2.595$\\$[-3.644,\,-2.527]$} & \shortstack{$-2.555$\\$[-3.565,\,-2.497]$} & \shortstack{$-2.521$\\$[-3.523,\,-2.462]$}
\\
\bottomrule
\end{tabular}
\end{table}

\section{Window means, shape, and energy scores}
For each record and its window weights, the variance across all window entries equals the mean variance within windows plus the variance of window means. The mean variance within windows accounts for 4.80 to 16.69\% of training variance in the original records, 99.14 to 99.68\% of their test variance, and 95.99 to 99.32\% of training variance in the additional records. These ranges use the weights for each record and partition. Energy score combines level and shape differences inside a Euclidean norm. Subtracting window means changes this distance and allows a separate assessment of window shape. At 44402, the 2016 height step occurs near the catalogue deployment date of 23 April, and the 2018 record falls within one catalogue deployment period. Identifying the cause of the height step requires evidence beyond those dates.

Let $P x=x-\bar x\mathbf 1$ remove each window's own mean. Apply this transformation separately to observations and model samples. The three energy scores use full windows $x$, windows after mean subtraction $Px$, and the following vector of time increments:
\begin{equation}
\Phi(x)=\left(\frac{x_{t+\ell}-x_t}{\sqrt{192-\ell}}\right)_{\ell\in\mathcal L,\;0\leq t<192-\ell},
\qquad \mathcal L=\{1,4,16,48,96\}.
\label{eq:increments}
\end{equation}
The vector of time increments has 795 coordinates. Since it includes adjacent differences, $\ker\Phi=\operatorname{span}\{\mathbf 1\}$ and $\operatorname{rank}\Phi=191$. It determines window shape up to an additive constant, including dependence across positions. Extra lags change the distance: writing $\Phi(x)=Ax$ gives $\sqrt{(x-y)^\top A^\top A(x-y)}$, in place of Euclidean distance between $Px$ and $Py$. Under finite first moments, the energy scores of $Px$ and $\Phi(x)$ are proper for their respective distributions. Both operations remove window means. Their scores assess the resulting shape distributions; the score on full windows also assesses mean levels and their dependence on shape. The two shape scores use different distances.

All DART scores use the dimensionless coordinates defined by the training mean and standard deviation. Multiplying by a record's positive training scale converts its energy scores to meters. Each of three independent prior streams samples 1,024 latent values and standard normal noise vectors; each condition multiplies the noise by its observation standard deviation. Within a record, the same latent values and noise vectors are used across conditions and before and after variance estimation. IW estimates use three separate streams of encoder samples. Scores are averaged over the three streams within each fit. Differences pair conditions with the same training seed and sampling stream. Ranges across training seeds describe variation among models, while ranges across sampling streams within a fit describe Monte Carlo variation.

\begin{table}[t]
\centering
\setlength{\tabcolsep}{2.4pt}
\small
\caption{DART energy score differences between skip and no skip at the same observation variance, in dimensionless coordinates. Scores use full windows, windows after subtraction of their own means, and the time increments in Eq.~\ref{eq:increments}. Negative values favor the skip. Each difference is the skip score minus the score without a skip, at base variance $\sigma_0^2$ or high variance $25\sigma_0^2$. The median and full range describe three differences between fits with the same training seed, after averaging streams within each fit. Maximum range width is the largest, over training seeds, of the maximum minus minimum difference across three common sampling streams. Original records use network weights chosen by validation and equal window weights. The 2018 records use final network weights and equal total window weight in each UTC start hour.}\label{tab:dart_contrasts}
\begin{tabular}{lllrrr}
\toprule
Record & Variance & Statistic & \shortstack{Full\\windows} & \shortstack{After mean\\subtraction} & \shortstack{Time\\increments}\\
\midrule

23461/2022 & Base & Median and range & \shortstack{0.369\\{[-1.21, 1.5]}} & \shortstack{0.0255\\{[-0.0553, 0.297]}} & \shortstack{0.00303\\{[-0.00289, 0.0262]}} \\
23461/2022 & Base & Maximum range width & 0.1585 & 0.03316 & 0.004819 \\
23461/2022 & High & Median and range & \shortstack{0.83\\{[0.773, 3.4]}} & \shortstack{0.0822\\{[-0.0275, 0.0868]}} & \shortstack{0.0111\\{[0.000783, 0.0154]}} \\
23461/2022 & High & Maximum range width & 0.4628 & 0.03795 & 0.005022 \\
44402/2016 & Base & Median and range & \shortstack{-0.153\\{[-1.27, 0.849]}} & \shortstack{-0.0814\\{[-0.144, 0.032]}} & \shortstack{-0.00857\\{[-0.0311, 0.0141]}} \\
44402/2016 & Base & Maximum range width & 0.2093 & 0.02672 & 0.004928 \\
44402/2016 & High & Median and range & \shortstack{0.115\\{[0.111, 0.172]}} & \shortstack{0.0313\\{[0.0256, 0.0371]}} & \shortstack{0.00319\\{[0.000627, 0.00757]}} \\
44402/2016 & High & Maximum range width & 0.4656 & 0.05392 & 0.006256 \\
46407/2011 & Base & Median and range & \shortstack{0.903\\{[-0.887, 0.936]}} & \shortstack{-0.0851\\{[-0.0892, -0.0797]}} & \shortstack{-0.00199\\{[-0.00548, -0.000733]}} \\
46407/2011 & Base & Maximum range width & 0.1399 & 0.04674 & 0.006484 \\
46407/2011 & High & Median and range & \shortstack{0.103\\{[-0.327, 0.111]}} & \shortstack{-0.0148\\{[-0.0668, 0.0736]}} & \shortstack{-0.000348\\{[-0.00275, 0.0147]}} \\
46407/2011 & High & Maximum range width & 0.2988 & 0.0161 & 0.002748 \\
52403/2020 & Base & Median and range & \shortstack{-0.407\\{[-0.836, 0.753]}} & \shortstack{-0.0297\\{[-0.0601, 0.0404]}} & \shortstack{-0.00367\\{[-0.00459, 0.0071]}} \\
52403/2020 & Base & Maximum range width & 0.1334 & 0.03162 & 0.004903 \\
52403/2020 & High & Median and range & \shortstack{1.18\\{[-0.604, 3.18]}} & \shortstack{-0.0156\\{[-0.0214, -0.0122]}} & \shortstack{-0.00419\\{[-0.00425, -0.00196]}} \\
52403/2020 & High & Maximum range width & 0.3057 & 0.05334 & 0.008483 \\
44402/2018 & Base & Median and range & \shortstack{-0.337\\{[-0.768, -0.119]}} & \shortstack{-0.345\\{[-0.479, -0.296]}} & \shortstack{-0.035\\{[-0.0523, -0.0334]}} \\
44402/2018 & Base & Maximum range width & 0.05621 & 0.04827 & 0.007899 \\
44402/2018 & High & Median and range & \shortstack{-0.222\\{[-0.322, -0.219]}} & \shortstack{-0.228\\{[-0.363, -0.151]}} & \shortstack{-0.0268\\{[-0.042, -0.0156]}} \\
44402/2018 & High & Maximum range width & 0.06129 & 0.04827 & 0.006903 \\
21415/2018 & Base & Median and range & \shortstack{-0.3\\{[-0.346, -0.227]}} & \shortstack{-0.372\\{[-0.393, -0.327]}} & \shortstack{-0.0564\\{[-0.058, -0.0467]}} \\
21415/2018 & Base & Maximum range width & 0.02322 & 0.02964 & 0.004697 \\
21415/2018 & High & Median and range & \shortstack{-0.253\\{[-0.341, -0.236]}} & \shortstack{-0.252\\{[-0.32, -0.223]}} & \shortstack{-0.0375\\{[-0.051, -0.0349]}} \\
21415/2018 & High & Maximum range width & 0.06139 & 0.0564 & 0.009434 \\
21418/2018 & Base & Median and range & \shortstack{-0.613\\{[-0.852, -0.266]}} & \shortstack{-0.453\\{[-0.626, -0.411]}} & \shortstack{-0.0577\\{[-0.0901, -0.0548]}} \\
21418/2018 & Base & Maximum range width & 0.05868 & 0.02613 & 0.003922 \\
21418/2018 & High & Median and range & \shortstack{-0.476\\{[-0.559, -0.433]}} & \shortstack{-0.501\\{[-0.544, -0.357]}} & \shortstack{-0.0685\\{[-0.074, -0.0524]}} \\
21418/2018 & High & Maximum range width & 0.03753 & 0.05253 & 0.008115
\\
\bottomrule
\end{tabular}
\end{table}

\subsection{Baseline distributions}
\label{sec:reference_distributions}
Baseline distributions use the mean and covariance of training windows, with equal window weights for the original records and equal total weight in each UTC start hour for 2018. Write $v=\sigma_0^2$. PPCA at the neural observation variance has covariance $WW^\top+vI$ or $WW^\top+25vI$, with the loadings fitted separately at each variance. The model uses at most 16 loading directions along the leading covariance eigenvectors, with magnitudes $\sqrt{\max(\lambda_j-v,0)}$ at variance $v$. The diagonal Gaussian sets each coordinate variance to the larger of its empirical variance and the noise variance. A separate PPCA fit estimates noise by averaging the remaining covariance eigenvalues, with a minimum of $10^{-4}$. This gives $10^{-4}$ for the original records and 21415/2018, approximately $1.10338\times10^{-4}$ for 21418/2018, and $1.20477\times10^{-4}$ for 44402/2018. If this estimate equals the neural observation variance, the two PPCA procedures give the same distribution.

Each baseline distribution supplies three random sampling streams of 1,024 samples. Differences across streams measure sampling variation for that one distribution. The baseline distribution is unchanged when the neural observation variance is replaced. Figure~\ref{fig:dart_ppca} compares each neural fit with PPCA at the same observation variance.

At the same observation variance, PPCA has lower mean energy scores than every one of the 84 neural fits on all three scores. The advantage remains after removing window means and when scoring time increments, so it also concerns variation within windows. PPCA obtains its parameters by covariance eigendecomposition \citep{TippingBishop1999ProbabilisticPCA}; the neural models use variational inference and a finite number of optimizer steps. These comparisons change both the model class and fitting procedure. Energy scores use unconditional samples generated from the latent prior and decoder, together with test observations; encoder importance weights affect only the log density estimates.

The training covariance helps explain why a linear baseline is useful here. Let $S=\sum_j\lambda_j u_j u_j^\top$ be covariance about the mean window, using the training weights, with eigenvalues in decreasing order. For each scoring transform $T\in\{I,P,\Phi\}$, the fraction $\sum_{j=1}^{16}\lambda_j\|Tu_j\|^2/\operatorname{tr}(TST^\top)$ exceeds 99.97\% on every record. Thus, the leading 16 directions account for almost all empirical covariance trace, including after mean subtraction and for time increments. Gaussianity, dependence beyond covariance, and convergence of neural training remain undetermined by this calculation.

\begin{figure}[H]
\centering
\includegraphics[width=\linewidth,height=0.81\textheight,keepaspectratio]{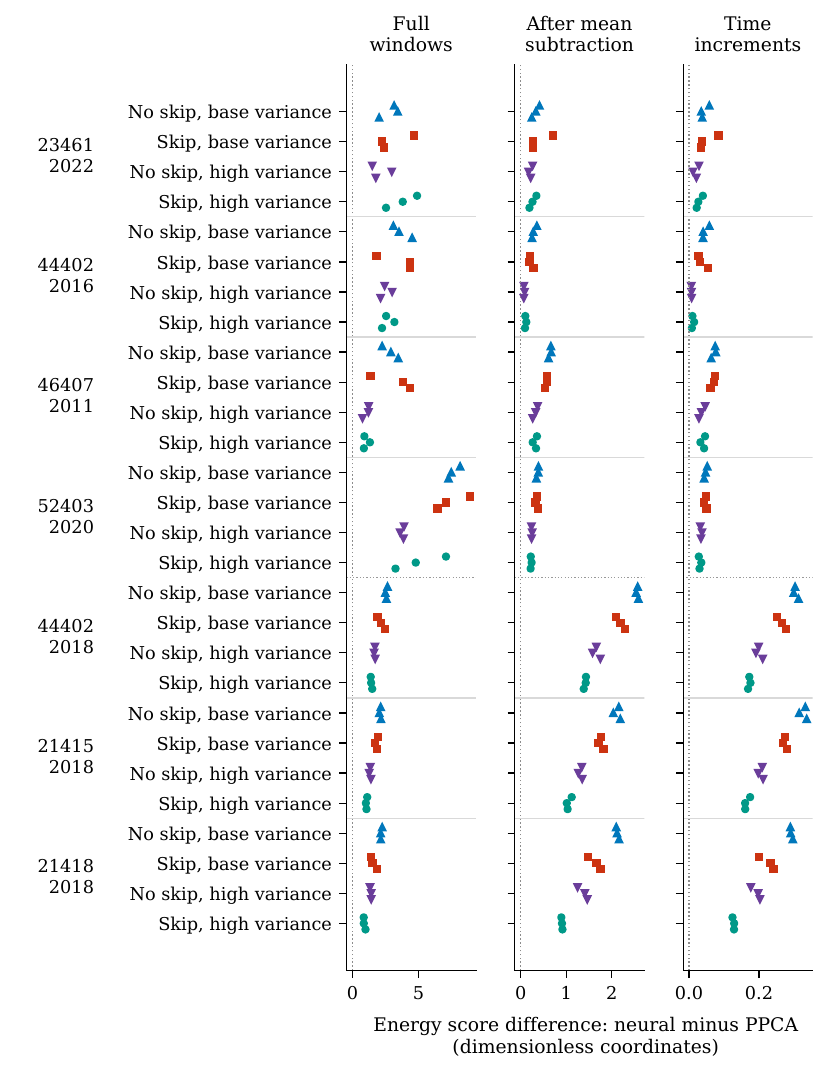}
\caption{Energy score of each VAE fit minus that of PPCA at the same observation variance. Positive values favor PPCA. Each row shows the three training seeds. Scores average three sampling streams within each fit. PPCA uses separate streams, and one PPCA fit supplies the baseline at each variance on each record. The three subplots use full windows, windows after subtraction of their own means, and the time increments in Eq.~\ref{eq:increments}, all in dimensionless coordinates. Base variance is $\sigma_0^2$ and high variance is $25\sigma_0^2$. Original records use network weights chosen by validation and equal window weights; 2018 records use final network weights and UTC start hour weights.}
\label{fig:dart_ppca}
\end{figure}

\subsection{Observation variance estimation}
\label{sec:variance_reset}
For the same encoder $q_\phi$ and decoder for observation means $g_\theta$, the Gaussian term that depends on variance, per coordinate, is $\tfrac12\log(2\pi v)+a/(2v)$, where $a=\E_{\mathrm{train}}\E_q\|x-g(Z)\|^2/d$. For $a>0$, this term is minimized at $v=a$. Three independent streams, each with 32 encoder samples per training window, give three estimates of $a$; their average gives the variance estimate. The largest range among those three estimates is 1.18\% of their average. This range measures Monte Carlo variation with the training windows and network weights kept the same.

The variance optimum minimizes reconstruction loss averaged over the encoder distribution on training windows. Energy score instead evaluates samples generated from the prior and decoder against test windows, using a different latent distribution, data partition, and loss. Negative ELBO is the sum of negative marginal log density and KL divergence from the encoder to the exact posterior. Changing observation variance also changes that posterior. A reduction in KL divergence from the encoder to the posterior can lower negative ELBO even when negative marginal log density increases.

The new observation variance is used both to generate noise and to evaluate the density, including its normalizer. The corresponding residual variance is $b^2\hat v$, with the Jacobian in Eq.~\ref{eq:density}. Encoder and decoder weights and outputs stay the same, and importance sampling continues to use the same encoder distribution. Across three common prior streams within a fit, the largest ranges of energy score changes are 0.00322 for full windows, 0.00280 after mean subtraction, and 0.000727 for time increments. Figure~\ref{fig:dart_variance_refit} distinguishes these sampling ranges from variation across training seeds, and Table~\ref{tab:dart_variance} gives the median and full range across the three seeds. These comparisons leave the effect on exact marginal likelihood and predictive calibration unresolved.

The two distance terms use the same samples and test weights as the full energy score. For transform $T$, mean score is $D_T-S_T$, where $D_T=\sum_i w_i M^{-1}\sum_m\|Ty_m-Tx_i\|$ and $S_T=[M(M-1)]^{-1}\sum_{m<n}\|Ty_m-Ty_n\|$. Table~\ref{tab:dart_energy_terms} compares their changes after averaging three streams within each fit. For every transform, the absolute change in $S_T$ is larger than that in $D_T$, giving the score change the opposite sign to the change in $S_T$. When both terms decrease and the decrease in $S_T$ exceeds that in $D_T$, samples are closer to observations on average but the score increases. The score rewards proximity through lower $D_T$ and spread through larger $S_T$; together these terms form a strictly proper score for each transformed distribution with a finite first moment \citep{GneitingRaftery2007ProperScoringRules}. Keeping the neural weights the same isolates the effect of observation variance on these scores.

\begin{table}[H]
\centering
\caption{Changes in the two distance terms of the energy score when observation variance is replaced, calculated as after minus before. $D_T$ averages distances from model samples to observations. $S_T$ is half the mean distance between distinct model samples and is subtracted from the first term. The displayed relation holds for every fit after averaging its three common sampling streams. Each variance group contains the same 42 fits for all three evaluations: full windows, windows after subtraction of their own means, and time increments.}
\label{tab:dart_energy_terms}
\begin{tabular}{lcll}
\toprule
Original variance & Fits & Changes in distance terms & Score after replacement\\
\midrule
$\sigma_0^2$ & 42 & $0<\Delta D_T<\Delta S_T$ & Lower\\
$25\sigma_0^2$ & 42 & $\Delta S_T<\Delta D_T<0$ & Higher\\
\bottomrule
\end{tabular}
\end{table}

\begin{figure}[H]
\centering
\includegraphics[width=\linewidth,height=0.81\textheight,keepaspectratio]{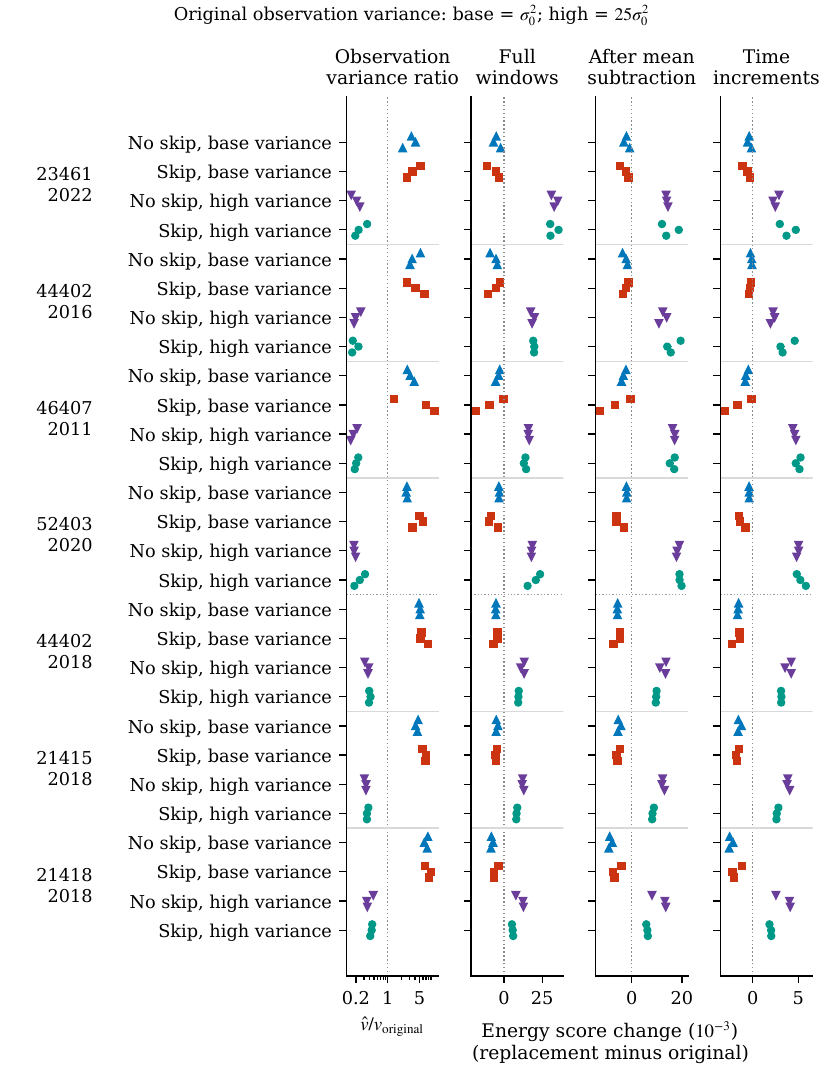}
\caption{Replacing observation variance in all 84 DART fits. For each record and condition, points run from top to bottom in the order of training seeds 42, 1729, and 31415. The variance ratio is $\hat v/v_{\mathrm{original}}$; $\hat v$ is estimated from three streams of encoder samples, and a ratio of one denotes equal variances. The variance ratio axis is logarithmic. Energy score columns show $10^3$ times the difference between the score after variance replacement and the score before replacement in dimensionless coordinates, averaged over three common prior streams. Negative values mean the score falls after variance replacement. Each thin bar runs from the smallest to the largest score change across those three streams for one fit. Original records use network weights chosen by validation and equal window weights; 2018 records use final network weights and UTC start hour weights.}
\label{fig:dart_variance_refit}
\end{figure}

\clearpage
\begingroup
\setlength{\tabcolsep}{3pt}
\small
\begin{longtable}{lllrrrr}
\caption{Observation variance estimated from training windows with encoder and decoder weights kept the same. The variance column gives $10^4\hat v$, where $\hat v$ is in squared dimensionless coordinates. Energy score columns give $10^3$ times the difference between the score after variance replacement and the score before replacement in dimensionless coordinates; negative values mean lower scores after variance replacement. Entries give the median across three training seeds, with the full range below, after averaging sampling streams within each fit. The original observation variance is $10^{-4}$ for base variance models and $0.0025$ for high variance models. Each model receives its own variance estimate, so conditions can have different values. The Skip column uses $\checkmark$ for $\alpha=0.8$ and $\times$ for $\alpha=0$.}\label{tab:dart_variance}\\
\toprule
Record & Skip & \shortstack{Original\\variance} & $10^4\hat v$ & $10^3\Delta$ \shortstack{Full\\windows} & $10^3\Delta$ \shortstack{After mean\\subtraction} & $10^3\Delta$ \shortstack{Time\\increments}\\
\midrule
\endfirsthead
\toprule
Record & Skip & \shortstack{Original\\variance} & $10^4\hat v$ & $10^3\Delta$ \shortstack{Full\\windows} & $10^3\Delta$ \shortstack{After mean\\subtraction} & $10^3\Delta$ \shortstack{Time\\increments}\\
\midrule
\endhead
23461/2022 & $\times$ & Base & \shortstack{3.32\\{[2.14, 4.09]}} & \shortstack{-5.13\\{[-6.91, -2.27]}} & \shortstack{-2.11\\{[-3.1, -0.803]}} & \shortstack{-0.395\\{[-0.559, -0.141]}} \\
23461/2022 & $\checkmark$ & Base & \shortstack{3.53\\{[2.63, 5.27]}} & \shortstack{-5.15\\{[-11, -3.29]}} & \shortstack{-2.25\\{[-4.58, -1.3]}} & \shortstack{-0.566\\{[-1.11, -0.32]}} \\
23461/2022 & $\times$ & High & \shortstack{5.18\\{[3.94, 6.08]}} & \shortstack{32.9\\{[31, 35.2]}} & \shortstack{13.9\\{[13.7, 14.4]}} & \shortstack{2.48\\{[2.22, 2.88]}} \\
23461/2022 & $\checkmark$ & High & \shortstack{5.74\\{[4.84, 8.83]}} & \shortstack{30.4\\{[30.2, 35.7]}} & \shortstack{13.7\\{[12.1, 18.8]}} & \shortstack{3.7\\{[2.97, 4.72]}} \\
44402/2016 & $\times$ & Base & \shortstack{3.45\\{[3.1, 5.24]}} & \shortstack{-5.3\\{[-9.19, -4.46]}} & \shortstack{-2.22\\{[-3.63, -1.67]}} & \shortstack{-0.105\\{[-0.246, -0.0806]}} \\
44402/2016 & $\checkmark$ & Base & \shortstack{4.1\\{[2.66, 6.45]}} & \shortstack{-5.23\\{[-10.7, -2.68]}} & \shortstack{-2.39\\{[-3.57, -1.3]}} & \shortstack{-0.281\\{[-0.395, -0.172]}} \\
44402/2016 & $\times$ & High & \shortstack{4.93\\{[4.55, 6.41]}} & \shortstack{18.4\\{[17.5, 19.9]}} & \shortstack{12.4\\{[10.9, 13.9]}} & \shortstack{2.23\\{[1.94, 2.4]}} \\
44402/2016 & $\checkmark$ & High & \shortstack{4.27\\{[4.14, 5.68]}} & \shortstack{19.7\\{[19.1, 19.8]}} & \shortstack{15.6\\{[14.1, 19.5]}} & \shortstack{3.28\\{[3.04, 4.6]}} \\
46407/2011 & $\times$ & Base & \shortstack{3.12\\{[2.7, 3.83]}} & \shortstack{-3.54\\{[-5.57, -2.85]}} & \shortstack{-3.3\\{[-4.05, -2.28]}} & \shortstack{-0.76\\{[-0.828, -0.486]}} \\
46407/2011 & $\checkmark$ & Base & \shortstack{6.91\\{[1.36, 10.8]}} & \shortstack{-9.52\\{[-18.3, -0.433]}} & \shortstack{-6.68\\{[-12.6, -0.418]}} & \shortstack{-1.68\\{[-3.08, -0.119]}} \\
46407/2011 & $\times$ & High & \shortstack{4.39\\{[3.84, 5.29]}} & \shortstack{15.9\\{[15.8, 16.5]}} & \shortstack{17.1\\{[16.3, 17.1]}} & \shortstack{4.59\\{[4.39, 4.73]}} \\
46407/2011 & $\checkmark$ & High & \shortstack{5.07\\{[4.72, 5.65]}} & \shortstack{14\\{[12.9, 14.5]}} & \shortstack{16.9\\{[15.2, 17.1]}} & \shortstack{5.13\\{[4.72, 5.24]}} \\
52403/2020 & $\times$ & Base & \shortstack{2.65\\{[2.56, 2.7]}} & \shortstack{-3.36\\{[-3.43, -3.07]}} & \shortstack{-2.01\\{[-2.08, -1.94]}} & \shortstack{-0.388\\{[-0.408, -0.386]}} \\
52403/2020 & $\checkmark$ & Base & \shortstack{4.99\\{[3.54, 6.07]}} & \shortstack{-8.48\\{[-9.78, -4.19]}} & \shortstack{-6.09\\{[-6.1, -3.05]}} & \shortstack{-1.37\\{[-1.49, -0.785]}} \\
52403/2020 & $\times$ & High & \shortstack{4.55\\{[4.52, 4.92]}} & \shortstack{18\\{[17.9, 18.5]}} & \shortstack{18.1\\{[17.8, 19]}} & \shortstack{5.03\\{[4.81, 5.05]}} \\
52403/2020 & $\checkmark$ & High & \shortstack{6.09\\{[4.62, 7.94]}} & \shortstack{20.8\\{[15.4, 23.6]}} & \shortstack{19.1\\{[19, 19.9]}} & \shortstack{5.23\\{[4.84, 5.83]}} \\
44402/2018 & $\times$ & Base & \shortstack{5.09\\{[4.89, 5.18]}} & \shortstack{-5.38\\{[-5.48, -5.3]}} & \shortstack{-5.64\\{[-5.72, -5.57]}} & \shortstack{-1.62\\{[-1.67, -1.55]}} \\
44402/2018 & $\checkmark$ & Base & \shortstack{5.55\\{[5.25, 7.79]}} & \shortstack{-4.42\\{[-6.96, -4.36]}} & \shortstack{-4.74\\{[-7.25, -4.57]}} & \shortstack{-1.44\\{[-2.25, -1.38]}} \\
44402/2018 & $\times$ & High & \shortstack{9.21\\{[7.83, 9.58]}} & \shortstack{13.1\\{[10.9, 13.2]}} & \shortstack{13.5\\{[11.2, 13.7]}} & \shortstack{4.22\\{[3.54, 4.22]}} \\
44402/2018 & $\checkmark$ & High & \shortstack{9.8\\{[9.71, 10.5]}} & \shortstack{9.4\\{[9.27, 9.62]}} & \shortstack{9.75\\{[9.6, 9.91]}} & \shortstack{3.12\\{[3.11, 3.15]}} \\
21415/2018 & $\times$ & Base & \shortstack{4.52\\{[4.1, 4.66]}} & \shortstack{-5.15\\{[-5.22, -4.06]}} & \shortstack{-5.34\\{[-5.47, -4.26]}} & \shortstack{-1.57\\{[-1.61, -1.25]}} \\
21415/2018 & $\checkmark$ & Base & \shortstack{6.81\\{[5.84, 7.06]}} & \shortstack{-5.39\\{[-5.67, -4.62]}} & \shortstack{-5.62\\{[-6.02, -4.77]}} & \shortstack{-1.76\\{[-1.87, -1.51]}} \\
21415/2018 & $\times$ & High & \shortstack{8.3\\{[7.68, 8.36]}} & \shortstack{12.2\\{[11.9, 12.9]}} & \shortstack{12.3\\{[12, 13]}} & \shortstack{3.85\\{[3.76, 4.05]}} \\
21415/2018 & $\checkmark$ & High & \shortstack{8.76\\{[8.73, 9.4]}} & \shortstack{8.17\\{[8, 8.65]}} & \shortstack{8.29\\{[8.15, 8.84]}} & \shortstack{2.64\\{[2.6, 2.82]}} \\
21418/2018 & $\times$ & Base & \shortstack{7.39\\{[6.39, 7.64]}} & \shortstack{-8.32\\{[-8.61, -7.26]}} & \shortstack{-8.75\\{[-9.1, -7.73]}} & \shortstack{-2.54\\{[-2.58, -2.17]}} \\
21418/2018 & $\checkmark$ & Base & \shortstack{8.19\\{[6.52, 8.87]}} & \shortstack{-6.44\\{[-6.53, -3.73]}} & \shortstack{-6.82\\{[-7.33, -4.11]}} & \shortstack{-2.08\\{[-2.23, -1.21]}} \\
21418/2018 & $\times$ & High & \shortstack{8.93\\{[8.68, 12]}} & \shortstack{12.3\\{[7.71, 12.7]}} & \shortstack{13.2\\{[8.09, 13.6]}} & \shortstack{4.04\\{[2.54, 4.11]}} \\
21418/2018 & $\checkmark$ & High & \shortstack{11.1\\{[10.4, 11.4]}} & \shortstack{5.64\\{[5.16, 5.96]}} & \shortstack{6.23\\{[5.8, 6.43]}} & \shortstack{2.02\\{[1.84, 2.03]}}
\\
\bottomrule
\end{longtable}
\endgroup

\clearpage
\subsection{Network weights, window weights, and variation through time}
An additional evaluation of the 2018 models uses network weights chosen by validation energy score and separate prior and IW streams. These samples differ from those used for final weights even when validation chooses step 10,000. The comparison with equal window weights reuses the primary scores at final weights. Both comparisons were specified before inspection of the new test results.

\begingroup
\setlength{\tabcolsep}{2.5pt}
\small
\begin{longtable}{lllrrrr}
\caption{Alternative network weights and window weights for the three 2018 records. Each entry gives the median energy score difference across three training seeds, with the full range, for conditions paired by seed. Negative values favor the skip. The first evaluation uses the network weights with the lowest validation energy score and gives each UTC start hour equal total window weight. The other evaluation uses final network weights and gives every window equal weight. Steps is the minimum and maximum Adam step count over both conditions and all seeds, or their common final count. Scores are averaged over streams within each fit. Evaluations at validation and final network weights use different samples, even when both use the final training step. Each difference compares skip with no skip at the same original observation variance.}\label{tab:dart_selected}\\
\toprule
Record & Variance & Evaluation & Steps & \shortstack{Full\\windows} & \shortstack{After mean\\subtraction} & \shortstack{Time\\increments}\\
\midrule
\endfirsthead
\toprule
Record & Variance & Evaluation & Steps & \shortstack{Full\\windows} & \shortstack{After mean\\subtraction} & \shortstack{Time\\increments}\\
\midrule
\endhead
44402/2018 & Base & \shortstack{Validation\\weights} & 6000--10000 & \shortstack{-0.359\\{[-0.748, -0.128]}} & \shortstack{-0.385\\{[-0.519, -0.298]}} & \shortstack{-0.0368\\{[-0.055, -0.0353]}} \\
44402/2018 & Base & \shortstack{Equal window\\weights} & 10000 & \shortstack{-0.337\\{[-0.768, -0.119]}} & \shortstack{-0.346\\{[-0.479, -0.296]}} & \shortstack{-0.035\\{[-0.0522, -0.0334]}} \\
44402/2018 & High & \shortstack{Validation\\weights} & 10000--10000 & \shortstack{-0.262\\{[-0.33, -0.241]}} & \shortstack{-0.239\\{[-0.37, -0.201]}} & \shortstack{-0.0288\\{[-0.0431, -0.0213]}} \\
44402/2018 & High & \shortstack{Equal window\\weights} & 10000 & \shortstack{-0.222\\{[-0.322, -0.219]}} & \shortstack{-0.227\\{[-0.363, -0.151]}} & \shortstack{-0.0267\\{[-0.0419, -0.0157]}} \\
21415/2018 & Base & \shortstack{Validation\\weights} & 8000--10000 & \shortstack{-0.328\\{[-0.334, -0.255]}} & \shortstack{-0.406\\{[-0.409, -0.317]}} & \shortstack{-0.06\\{[-0.0612, -0.0448]}} \\
21415/2018 & Base & \shortstack{Equal window\\weights} & 10000 & \shortstack{-0.299\\{[-0.35, -0.211]}} & \shortstack{-0.371\\{[-0.377, -0.331]}} & \shortstack{-0.0555\\{[-0.0562, -0.0473]}} \\
21415/2018 & High & \shortstack{Validation\\weights} & 10000--10000 & \shortstack{-0.267\\{[-0.335, -0.228]}} & \shortstack{-0.266\\{[-0.329, -0.204]}} & \shortstack{-0.0397\\{[-0.0526, -0.0314]}} \\
21415/2018 & High & \shortstack{Equal window\\weights} & 10000 & \shortstack{-0.248\\{[-0.345, -0.225]}} & \shortstack{-0.247\\{[-0.324, -0.212]}} & \shortstack{-0.0364\\{[-0.0519, -0.0328]}} \\
21418/2018 & Base & \shortstack{Validation\\weights} & 9000--10000 & \shortstack{-0.638\\{[-0.82, -0.282]}} & \shortstack{-0.463\\{[-0.633, -0.409]}} & \shortstack{-0.0587\\{[-0.0918, -0.0535]}} \\
21418/2018 & Base & \shortstack{Equal window\\weights} & 10000 & \shortstack{-0.61\\{[-0.856, -0.262]}} & \shortstack{-0.45\\{[-0.631, -0.407]}} & \shortstack{-0.0573\\{[-0.091, -0.0541]}} \\
21418/2018 & High & \shortstack{Validation\\weights} & 10000--10000 & \shortstack{-0.462\\{[-0.577, -0.453]}} & \shortstack{-0.508\\{[-0.553, -0.366]}} & \shortstack{-0.0689\\{[-0.0755, -0.0533]}} \\
21418/2018 & High & \shortstack{Equal window\\weights} & 10000 & \shortstack{-0.473\\{[-0.557, -0.424]}} & \shortstack{-0.499\\{[-0.535, -0.355]}} & \shortstack{-0.0683\\{[-0.0724, -0.0521]}}
\\
\bottomrule
\end{longtable}
\endgroup

Scores are grouped into intervals of 14 days and 28 days from the start of the first eligible training window. Each test window belongs to the interval containing its start time, even if the window extends beyond the interval. An interval's total weight is the sum of its window weights. Its weighted mean score multiplied by that total is its contribution to the overall mean, and the contributions sum to the overall mean. Intervals with few windows remain included. Figures~\ref{fig:dart_intervals_original} and~\ref{fig:dart_intervals_additional} show the differences for each interval and training seed, together with the number and total weight of its test windows. These points describe the available fits and intervals. Population confidence intervals would require a method that accounts for serial dependence and sparse coverage.

Adjacent 2018 windows start four hours apart and share 176 of their 192 observations. Split exclusions keep source observations in separate partitions, while overlap and recurring oscillations create dependence within partitions. The three scores all evaluate these same windows. For full windows, each of 21415 and 21418 has one training seed with positive differences in one interval of 14 days at both observation variances. The full test partitions and every interval of 28 days containing data favor the skip for all three seeds. The positive differences at 21415 occur in an interval with one window, which carries 0.44\% of test weight. At 21418, the corresponding interval contains ten windows and 5.26\% of test weight.

\begin{figure}[H]
\centering
\includegraphics[width=\linewidth,height=0.81\textheight,keepaspectratio]{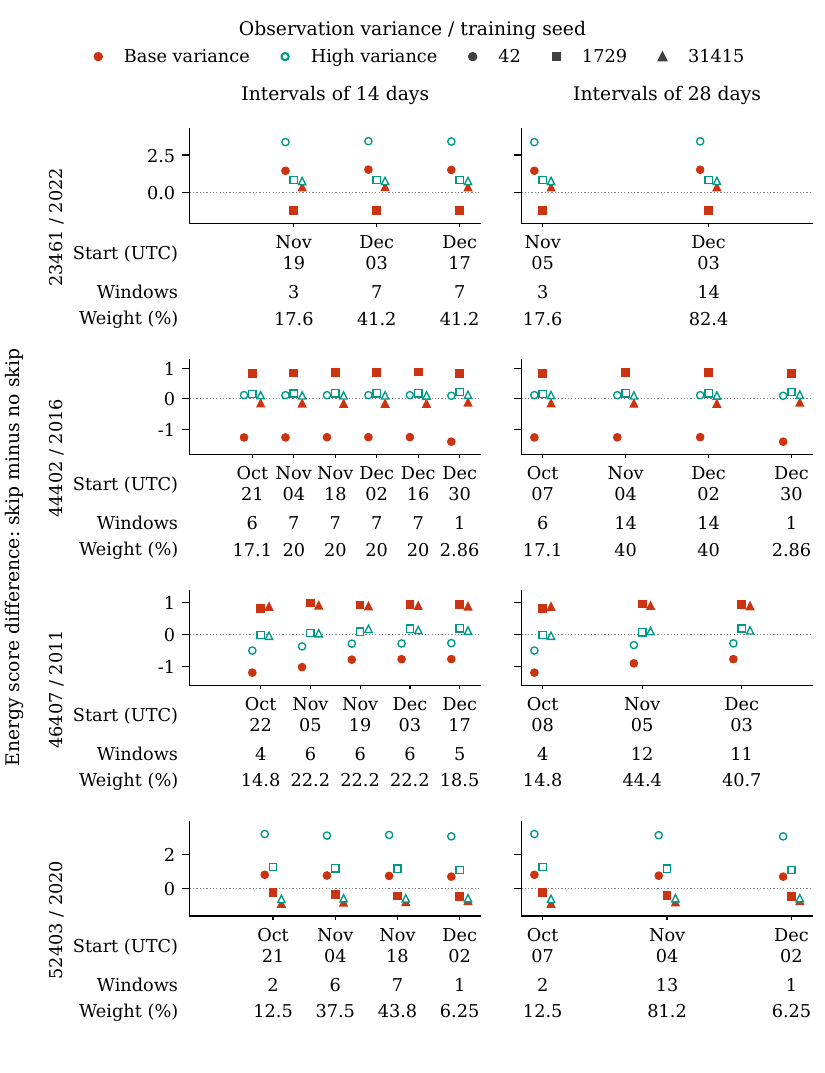}
\caption{Energy score differences over time for the original DART records. Each point gives the skip score minus the score without a skip at the same observation variance for one training seed, after averaging three sampling streams within each fit. Negative values favor the skip. Scores use full windows in dimensionless coordinates. The columns group test windows by start time into intervals of 14 or 28 days. Dates give the interval start in UTC. Horizontal spacing separates the seed markers around each interval start. The coverage below each subplot gives window counts and fractions of total test weight. Marker shapes identify training seeds. Filled points denote base variance $\sigma_0^2$ and open points high variance $25\sigma_0^2$.}
\label{fig:dart_intervals_original}
\end{figure}

\begin{figure}[H]
\centering
\includegraphics[width=\linewidth,height=0.81\textheight,keepaspectratio]{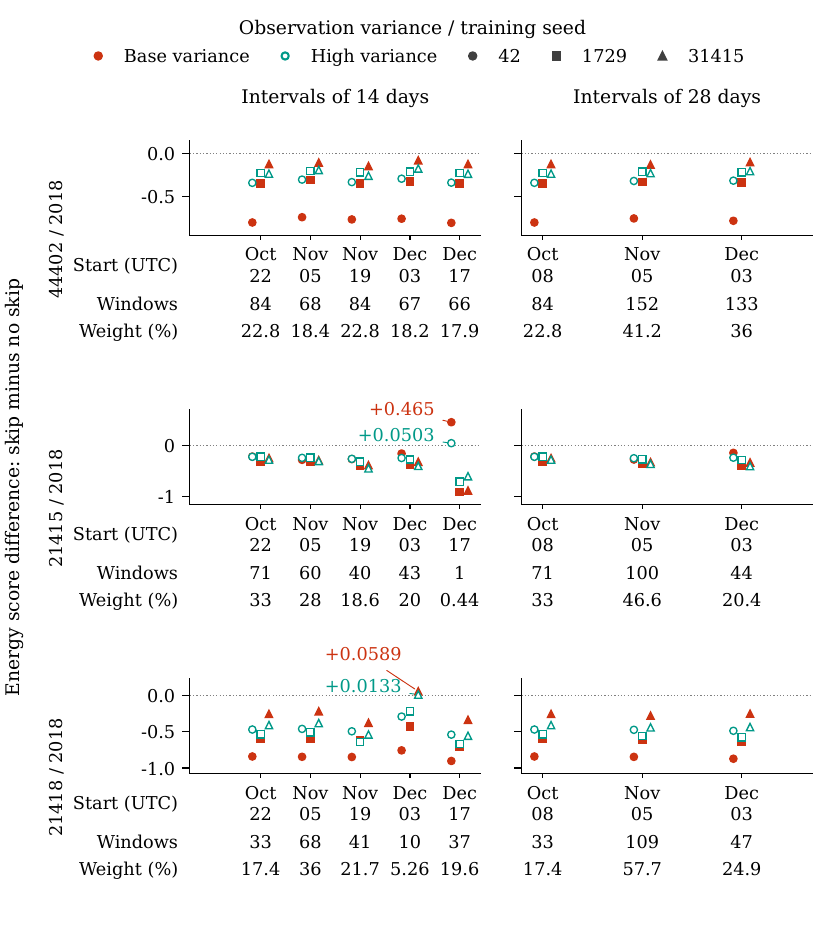}
\caption{Energy score differences over time for the additional DART records, with the definitions in Figure~\ref{fig:dart_intervals_original}. Interval scores use the global UTC start hour weights of their test windows. Counts include overlapping windows. For each of 21415 and 21418, one training seed gives positive differences over 14 days in an interval with little total test weight. At both observation variances, the skip has lower scores for all three seeds in every interval of 28 days containing test windows in these two records.}
\label{fig:dart_intervals_additional}
\end{figure}

\section{Feature ECDFs and spectra}
Figures~\ref{fig:dart_ecdf_historical} to~\ref{fig:dart_spectra_fresh} use the models, prior samples, and original observation variances from the DART energy score evaluation. Training and test windows have weights computed for their own partitions: equal window weights for the original records and equal total weight in each UTC start hour for 2018. Model samples have equal weight. For each neural fit, feature empirical cumulative distribution functions (ECDFs) and power spectra are calculated for three streams of 1,024 samples and then averaged. The curve at each feature value or frequency is the median across three training seeds, and the shading covers the full range. The median curve can take values from different fits.

Here PPCA uses the noise variance estimated from the remaining training covariance eigenvalues, subject to the minimum in Appendix~\ref{sec:reference_distributions}, and averages curves from one fit over three sampling streams.

\begin{figure}[!t]
\centering
\includegraphics[width=\linewidth,height=0.81\textheight,keepaspectratio]{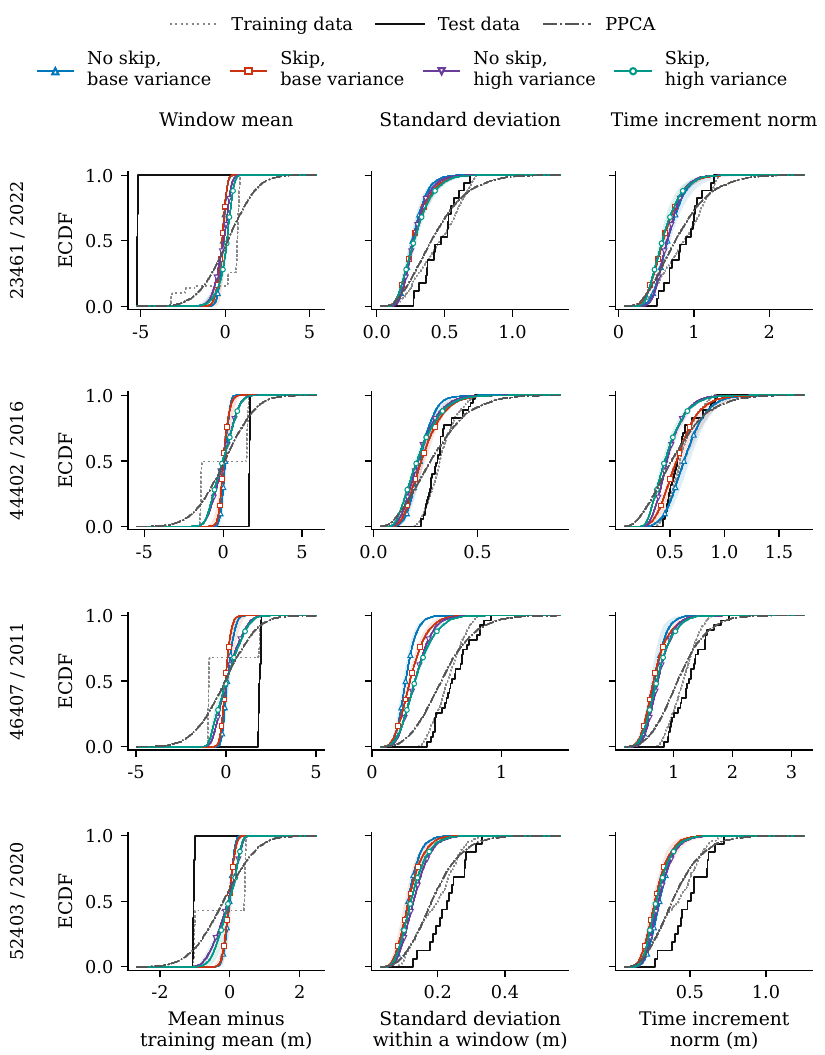}
\caption{Feature ECDFs for the four original DART records, comparing training and test observations with noisy prior samples from the four neural conditions and PPCA. Observed windows have equal weights. Each neural fit uses the network weights chosen by validation and three streams of 1,024 prior samples. After averaging the three ECDFs within each fit, the curve gives the median at each feature value across three fits per condition, and shading spans the full range across training seeds.}
\label{fig:dart_ecdf_historical}
\end{figure}

The ECDFs describe three window features: the difference between the window mean and training mean, the standard deviation within a window, and the Euclidean norm of the time increments in Eq.~\ref{eq:increments}. Multiplication by the training scale expresses each feature in meters. For feature values $a_i$, the ECDF is $\widehat F(t)=\sum_i\tilde w_i\mathbf{1}\{a_i\le t\}$. Weights sum to one separately within each observation partition or model sample set, and all ties at exact feature values remain. Each curve describes one feature's distribution. The same feature distributions can arise from different distributions of complete windows, including different dependence between mean and shape.

\begin{figure}[H]
\centering
\includegraphics[width=\linewidth,height=0.81\textheight,keepaspectratio]{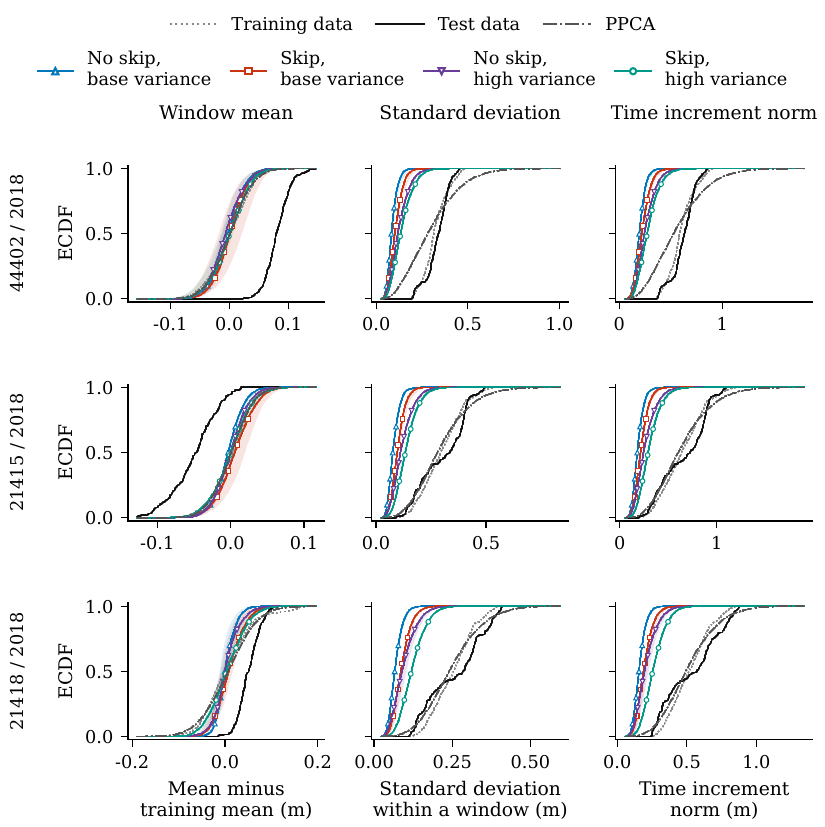}
\caption{Feature ECDFs for the three additional DART records. Training and test windows each use the UTC start hour weights of their own partition; model samples have equal weight. For each neural condition, three sampling streams are averaged within each final fit. The curve at each feature value is the median across three fits, with shading spanning the full range across training seeds. The PPCA curve averages three streams from one baseline fit.}
\label{fig:dart_ecdf_fresh}
\end{figure}

For each spectrum, the window mean is subtracted from 192 observations and every position receives equal weight. The discrete Fourier transform (DFT) gives a one-sided power spectrum by combining positive and negative frequencies: interior positive bins count twice and the Nyquist bin once. The 96 positive bins span 0.5 to 48 cycles/day. Power is in square meters per bin, and its sum equals the window variance; dividing by the bin width gives power spectral density. Powers are averaged in linear units before plotting on logarithmic axes.

\begin{figure}[H]
\centering
\includegraphics[width=\linewidth,height=0.81\textheight,keepaspectratio]{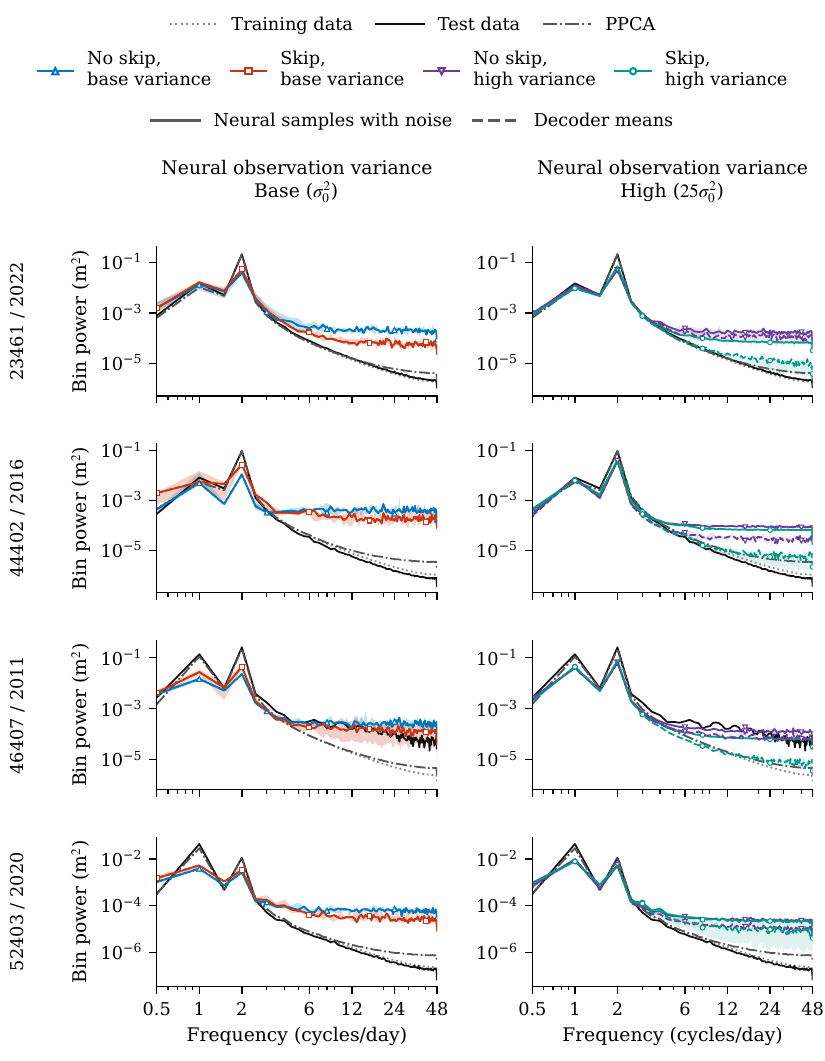}
\caption{DFT bin powers within windows for the four original records. The columns at base variance use $\sigma_0^2$, and those at high variance use $25\sigma_0^2$. Each column includes spectra of training and test observations and the same PPCA fit. Solid neural curves include observation noise; dashed curves use decoder means at prior latent samples. Neural models use network weights chosen by validation. Bin powers are averaged over three streams in linear units within each fit; curves and shading then show the median and full range across training seeds at each frequency. Subtracting each window's mean before the DFT removes variation in window levels. Both axes are logarithmic, and the 0.5 cycle per day bin spacing leaves nearby tidal constituents unresolved.}
\label{fig:dart_spectra_historical}
\end{figure}

\clearpage
The median across training seeds is taken at each frequency separately, so the sum of this median spectrum can differ from the median variance across fits. Dashed neural curves use decoder means evaluated at prior latent samples. The spectra describe variation within windows after removing their mean levels. Tidal oscillations remain, with nearby constituents unresolved at the 0.5 cycle per day bin spacing.

\begin{figure}[H]
\centering
\includegraphics[width=\linewidth,height=0.81\textheight,keepaspectratio]{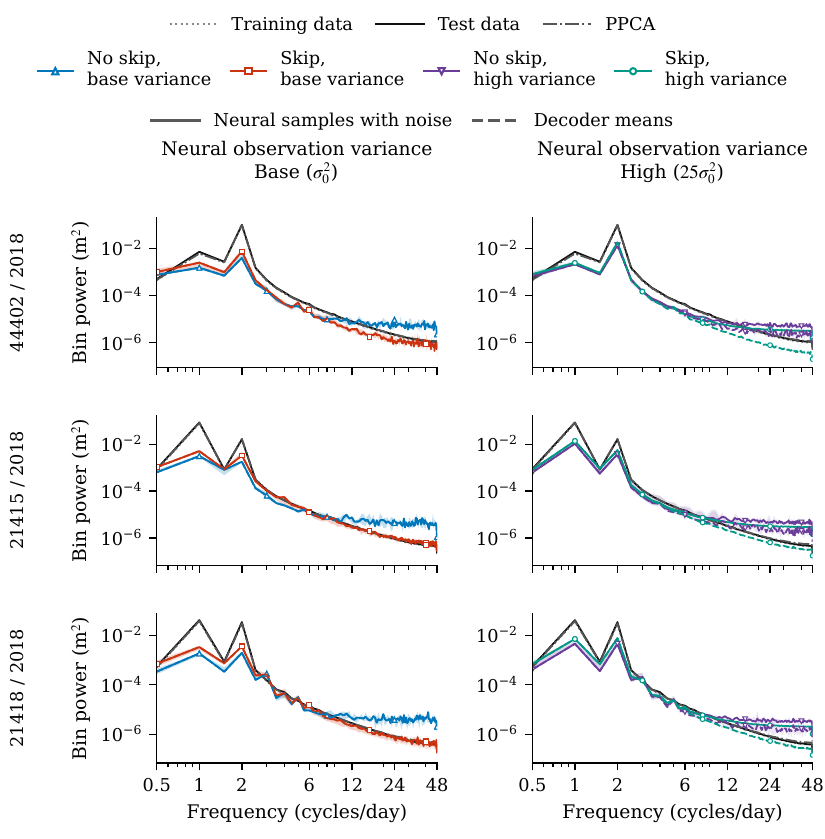}
\caption{DFT bin powers within windows for the three additional records. Training and test spectra use UTC start hour weights computed separately for each partition. Neural spectra use final fits at their original observation variance. Solid curves include observation noise; dashed curves use decoder means at prior latent samples. Within each fit, bin powers are averaged over three streams in linear units. Curves and shading then give the median and full range across training seeds at each frequency. Both columns use the same PPCA fit. Each window's mean is subtracted before the DFT. Both axes are logarithmic, and nearby tidal constituents remain unresolved at the 0.5 cycle per day bin spacing.}
\label{fig:dart_spectra_fresh}
\end{figure}

\clearpage
On the 2018 records, neural prior samples vary less within windows than the observations: both their standard deviations and norms of time increments are smaller. Their spectra contain less power in the prominent peaks at low frequencies and, in several conditions, more power at high frequencies. PPCA reproduces more of the variation, although its feature distributions also differ from the data; at 44402/2018, it gives a broader distribution of window standard deviations. Model class, inference method, and fitting procedure differ between PPCA and the neural fits, so their separate contributions remain unresolved.

\section{Records, tides, and HEIGHT}
\label{sec:scope}
The seven DART records cover six sites, including two years from 44402. The original records contribute 95 test windows and the additional records contribute 773 overlapping test windows. Serial dependence persists within these observations. The comparisons describe the available station records under the data selection and fitting procedures above.

The four original DART records were chosen and inspected before this evaluation. NCEI gives deployment dates of April 23, 2016 for 44402 \citep{NCEIDART44402}, May 29, 2011 for 46407 \citep{NCEIDART46407}, and August 13, 2020 for 52403 \citep{NCEIDART52403}. A deployment can involve sensor replacement or recalibration. All three dates precede their training cutoffs, so these training partitions span deployments and validation and test use later observations. Establishing the causes of these height changes requires evidence beyond deployment dates. Drift, environmental variability, data availability, and the reporting types included can also change the distribution within a deployment period.

In each sequence of consecutive routine observations in an original record, training windows start four hours apart and evaluation windows 48 hours apart. Over 48 hours, the solar semidiurnal S2 phase repeats, the lunisolar diurnal K1 phase nearly repeats, and the lunar semidiurnal M2 phase advances appreciably; a segment break can change the starting phase. Normalization counts each eligible training row once, including rows in fragments too short for a window. Overlapping training windows reuse observations, and test windows receive equal weights. Split exclusions keep source observations in separate partitions, but tidal and longer temporal dependence can extend across the splits.

HEIGHT expresses bottom pressure as equivalent water column height after compensation for sensor temperature effects, using 670 mm per pound per square inch absolute (PSIA) \citep{NDBCDARTDescription,Meinig2005DARTII}. This nominal pressure conversion gives a quantity distinct from surveyed instrument depth and surface elevation. NOAA describes measurements every 15 seconds and routine spot values at 15 minute intervals. Meinig et al.\ describe nominal integration over 15 seconds and separately call the transmitted height an ``average'' \citep[Sections 3.3.1, 3.3.2, and 3.3.8]{Meinig2005DARTII}. The firmware and timestamp conventions of a particular annual record are needed to apply those system descriptions to it. The routine windows used here have observations 900 s apart. Omitting event reporting types can affect which physical conditions the analysis includes.

The 900 s spacing gives a Nyquist frequency of 48 cycles/day, corresponding to a period of 30 minutes. A DFT of 192 observations has bins 0.5 cycles/day apart. Peaks near diurnal and semidiurnal frequencies can be compared, but nearby constituents remain unresolved: the M2/S2 and K1/O1 beat periods are about 14.8 and 13.7 days \citep{NOAA2000TideGlossary}. O1 is the lunar diurnal constituent. Harmonic fitting estimates amplitudes and phases at known constituent frequencies; the DFT plots show power in frequency bins. The reliability of harmonic fitting depends on record length, constituent separation, and noise \citep{Parker2007TidalAnalysis}. Coefficients from a longer calibration record can be applied to shorter windows.

Instrument response and sampling impose separate limits. \citet{EbleGonzalez1991BottomPressure} measured bottom pressure in the deep ocean at 64 samples per hour and described sensitivity at the millimeter scale in principle for oscillations longer than a few minutes; drift and background noise limited practical resolution. Deep water attenuates the bottom pressure response to ordinary wind waves and swell \citep[Section 2.2.1]{Meinig2005DARTII}. Interpreting power at high frequencies requires both the pressure response at the instrument's depth and the effects of finite sampling. The experiments here assess distributions of routine windows. Evaluating noise removal, anomaly classification, or operational warning would require independent evidence about the relevant events or measurement errors.

\end{document}